\documentclass[10pt,twocolumn,letterpaper]{article}

\usepackage{cvpr}              

\usepackage[table,dvipsnames]{xcolor}

\usepackage{amsthm} 
\newtheorem{lemma}{Lemma} 
\newtheorem{theorem}{Theorem} 

\definecolor{cvprblue}{rgb}{0.21,0.49,0.74}
\usepackage[pagebackref,breaklinks,colorlinks,citecolor=cvprblue]{hyperref}

\definecolor{top1}{gray}{0.58}  
\definecolor{top3}{gray}{0.72}  
\definecolor{top5}{gray}{0.87} 

\usepackage{booktabs}
\usepackage{makecell}

\def\paperID{*****} 
\def\confName{VISAPP\xspace}
\def\confYear{2026\xspace}

\title{Towards Understanding 3D Vision: the Role of Gaussian Curvature}

\author{
    Sherlon Almeida da Silva¹²
\and
    Davi Geiger²
\and
    Luiz Velho³
\and
    Moacir Antonelli Ponti¹
\and
    {\normalsize ¹Instituto de Ciências Matemáticas e de Computação, Universidade de São Paulo (ICMC-USP)} \\
    {\normalsize ²Courant Institute of Mathematical Sciences, New York University (NYU)} \\
    {\normalsize ³Instituto de Matemática Pura e Aplicada (IMPA)}
\and
     {\tt\small \{sherlon.a, dg1\}@nyu.edu},
     {\tt\small lvelho@impa.br},
     {\tt\small moacir@icmc.usp.br}
}

\begin{document}

\maketitle

\begin{abstract}
Recent advances in computer vision have predominantly relied on data-driven approaches that leverage deep learning and large-scale datasets. Deep neural networks have achieved remarkable success in tasks such as stereo matching and monocular depth reconstruction. However, these methods lack explicit models of 3D geometry that can be directly analyzed, transferred across modalities, or systematically modified for controlled experimentation. We investigate the role of Gaussian curvature in 3D surface modeling. Besides Gaussian curvature being an invariant quantity under change of observers or coordinate systems, we demonstrate using the Middlebury stereo dataset that it offers a sparse and compact description of 3D surfaces. Furthermore, we show a strong correlation between the performance rank of top state-of-the-art stereo and monocular methods and the low total absolute Gaussian curvature. We propose that this property can serve as a geometric prior to improve future 3D reconstruction algorithms.
\end{abstract}
\section{Introduction}
\label{sec:intro}
Vision is a fundamental sensory modality that allows us to perceive and interpret the structure of our surroundings, playing a critical role in numerous robotics applications. However, purely data-driven deep learning approaches -- despite their success in tasks such as depth estimation -- often lack explicit and transferable geometric representations.


We emphasize that \textit{understanding} or \textit{explaining} the visual world is not the same scientific goal as developing algorithms purely for \textit{prediction}. For example, a predictive model might use visual video data and machine learning to accurately forecast the trajectory of a falling object. Although such a system may yield precise results, it does not necessarily uncover or convey the underlying physical law: namely, \( F = ma \), where \( F \) is the gravitational force and \( a = 9.8\, \mathrm{m/s^2} \) is the acceleration due to gravity.  More generally, explanatory models aim to reveal the underlying structure of the world. It is commonly believed that such models offer greater \textbf{simplicity} and \textbf{generalization power}, as exemplified in the classical laws of physics where \( F = ma \) is applied to other physical scenarios (not just objects falling under gravity).



The world around us is shaped both by natural processes and by human-made structures, which are commonly represented in stereo-vision by depth and disparity measurements. These metrics can accurately describe 3D geometry but depend heavily on the observer's position and viewpoint, limiting their robustness. In contrast, Gaussian curvature is a purely local geometric property of a surface that remains invariant under changes in viewpoint, making it ideal for robust 3D surface analysis.


Our main contribution is on \textit{Foundational Vision Understanding} by deepening our understanding of 3D scene geometry, paving the way for more \textbf{interpretable}, \textbf{generalizable}, and \textbf{reliable} vision systems. This is an analytical work grounded on data from state-of-the-art (SOTA) Deep Learning (DL) stereo algorithms. We investigate their benchmarking performance and how this insight about the importance of Gaussian curvature description of the 3D world is implicitly captured by stereo algorithms with indoor scenes on Middlebury dataset. We show that (i) Gaussian curvature is sparsely distributed across natural 3D surfaces; (ii)   Low Gaussian curvature magnitude can serve as a prior for regularizing or modeling 3D surface geometry; (iii) This prior may be implicitly captured by current SOTA algorithms, but it is not explicit as an independent module and so we cannot use it elsewhere; (iv) Low Gaussian curvature magnitude enables the definition of a novel unsupervised evaluation metric for assessing the quality of 3D surface reconstruction algorithms.

The contributions fall under the following two umbrella concepts:

\noindent-- \textbf{Scene Understanding:}  The identification of quantities that are zero or nearly zero across most of a scene -- thus encoding data with a minimal set of active components -- forms the foundation of sparse representation. We show in Section~\ref{sec:sparse-representation} that the Gaussian curvature magnitude fulfills this role for 3D data representations of indoor environments. 

\noindent-- \textbf{Explainable AI:} In Section~\ref{sec:GC-metric} we empirically verify that the SOTA stereo algorithms seem to apply low Gaussian curvature magnitude priors. However, it is unclear whether or where these priors are explicitly used, making it impossible to isolate reusable algorithmic modules, in analogy to modules that perform feature extraction, from current depth reconstruction methods.

A potential immediate application of this explainability study is to use low Gaussian curvature  magnitude as an unsupervised metric to regularize models and/or evaluate 3D reconstruction algorithms.

\section{Previous Work}

We begin by highlighting the influential book by Koenderink~\cite{Koenderink90}, in which a differential geometry framework is used to provide a rich theoretical foundation for understanding surface geometry in vision. However, it does not address the construction of a sparse representation nor explore Gaussian curvature as a key feature for computer vision. The field of sparse image representation gained prominence in image analysis through the work of Field and Olshausen~\cite{Field94,OlshField97}, who highlighted the statistical regularities present in natural images and proposed sparse coding models inspired by the processing mechanisms of the visual cortex. Since then, these ideas have been extended to various visual domains, including, but not limited to, texture modeling~\cite{ZhuMumford98}, illumination and shape analysis~\cite{Yuille95,BasriJacobs2001}, template matching~\cite{Geigeretal95}, and image restoration~\cite{Mairal2009}. 

Some work worth mentioning as they address some aspects of the topics presented here. In studying visual illusions, Ishikawa and Geiger~\cite{ishikawa2006illusory} noted that humans may have low Gaussian curvature priors when reconstructing surfaces. A recent application of imposing the zero Gaussian curvature on man-made CAD reconstruction indicates the value of such extreme case (zero Gaussian curvature) for designing industrial applications \cite{NeurCADRecon2024}. Also, Guo~\cite{guo20103d} introduced the Gaussian Curvature Co-occurrence Matrix, combining curvature with co-occurrence statistics for 3D shape representation. Zhong and Qin~\cite{zhong2014sparse} proposed a sparse approximation for 3D shapes using spectral graph wavelets to capture local geometric details. Ververas et al.~\cite{ververas2024sags} apply a curvature-based densification step to populate an underrepresented area. These studies highlight the potential of integrating curvature and sparse representations in 3D geometry analysis.

In contrast to previous approaches to 3D surface understanding, which often focus on quantities defined within specific coordinate frames or based on appearance~\cite{zhou2024comprehensive,wang2024dust3r}, our work seeks a geometric quantity that is invariant to the choice of observer or coordinate system, a property that is especially desirable for 3D surface analysis. To the best of our knowledge, there have been no prior studies applying sparse representation principles to Gaussian curvature in 3D geometry.


\section{Gaussian Curvature (GC)}
\label{sec:Gaussian-Curvature}
We begin by briefly reviewing the concept of GC and explaining its relevance to our study. Next, we discuss existing methods for estimating GC. Finally, we describe the synthetic 3D scenes we generated to enable controlled curvature analysis and benchmarking.

\subsection{Brief Review}

 \textbf{GC}  of a smooth surface at a given point is given by \cite{do2016differential}:
\begin{align}
  K = \kappa_1 \kappa_2, 
  \label{eq:Gaussian-Curvature}
\end{align}
where \( \kappa_1 \) and \( \kappa_2 \) are the {principal curvatures} of a surface, i.e.,  the maximum and minimum normal curvatures. 
The sign of the GC indicates the local shape of the surface: (i) Positive GC (\( K > 0 \)) and the surface is locally convex. Both principal curvatures have the same sign, e.g. sphere of radius \( r \) has constant curvature $K = \frac{1}{r^2}$; (ii) Zero GC (\( K = 0 \)) and one of the principal curvatures is zero, e.g. plane or cylinder; (iii) Negative GC (\( K < 0 \)) and the principal curvatures have opposite signs, e.g. a hyperboloid of one sheet or the inner region of a torus. Note that a GC has dimensions inverse to the square of distance.
Gauss's Theorema Egregium states that \textbf{GC is an intrinsic property} of a surface. 



This invariance makes GC particularly valuable for computer vision and geometric modeling, as it provides a consistent descriptor regardless of viewpoint or deformation. The Gauss–Bonnet theorem further reinforces its significance by relating the integral of GC over a surface to its Euler characteristic, thereby bridging local geometric information with global topological structure. Moreover, previous studies have shown that low GC plays a role in human shape understanding (see \cite{ishikawa2006illusory}). 


\subsection{Methods to estimate GC from data}
In order to estimate a GC  we use the formula \cite{do2016differential}:
\begin{align}
    K={\frac {\det(\mathrm {I\!I} )}{\det(\mathrm {I} )}}\, . 
    \label{eq:Gaussian-curvature-formula}
\end{align}
where the fundamental forms I and II can be obtained from a parametric surface $\mathbf{Z}(u, v)$, its derivatives   $\mathbf{Z}_u(u, v),\mathbf{Z}_v(u, v)$ and second derivatives $\mathbf{Z} _{uu}, \mathbf{Z} _{uv}, \mathbf{Z} _{vv}$.
More details of the formula are in the supplementary material.

We construct the parametrized surface from the depth data. More precisely, given a set of \( n \) depth measurements \( \{ d_i(x_i, y_i) \}_{i=1}^n \), we represent the surface as a 3D point cloud \( \{ (X_i, Y_i, d_i(x_i, y_i))_{i=1}^n \} \) where $X_i,Y_i$ are obtained by inverting the projection, that is, $X_i = \frac{(x_i - c_x)}{f_x} d_i(x_i, y_i) $ and $Y_i = \frac{(y_i - c_y)}{f_y} d_i(x_i, y_i) $, before assigning a parametrized surface $\mathbf{Z}(u, v)$. Since numerical differentiation tends to amplify noise, especially when applied to noisy depth data, we apply smoothing when necessary. Importantly, we smoothed the 3D point cloud with Gaussian smoothing over each 3D component, $X, Y, Z$, independently. Note that one should not smooth the raw depth map, as the depth values are not uniformly sampled with respect to surface distances, and the smoothing in the image space does not adequately account for the true geometry (see Figures~\ref{fig:distance_discretization} and~\ref{fig:height_map_x_3D_point_cloud}). 

\begin{figure}
    \begin{center}
    \includegraphics[width=1.\columnwidth]{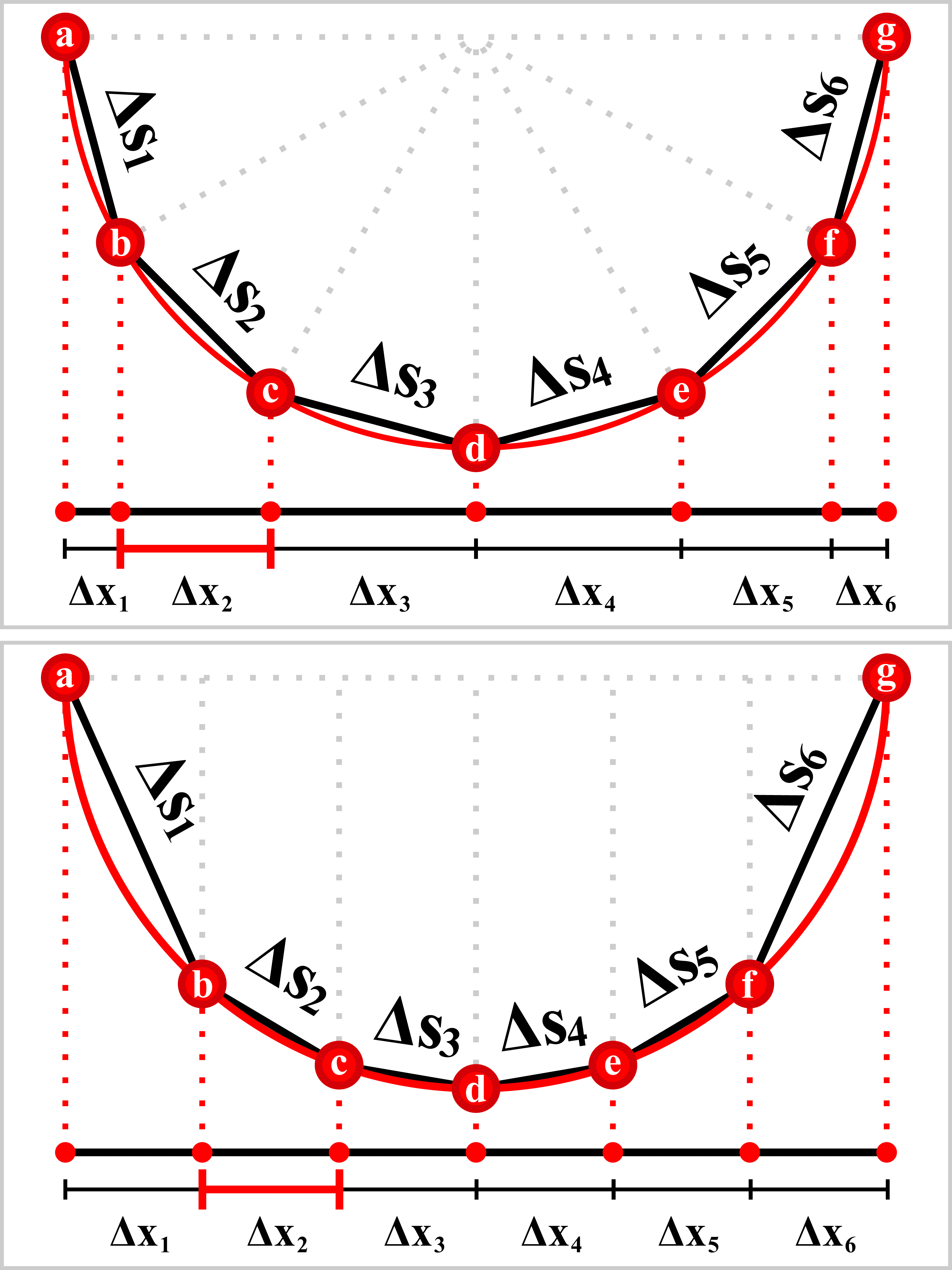}
    \end{center}
        \caption{\textbf{a,b,...,g} are sample points of a surface $z(x,y)$ along a slice $z(x, y_0)$. GC require calculations with surface distances between points, here denoted by $\{\Delta s_i; i=1\hdots 6\}$ and shown on a curve, a slice of a surface.  a. Top: The distances are constant, i.e.,  $\Delta s_1=\Delta s_2=\hdots =\Delta s_6 $. Resulting in non-uniform projected distances $\Delta x_i; i=1\hdots 6$. b. Bottom: The projected distances $\Delta x_i; i=1\hdots 6$ are constant, but then the distances $\Delta s_i; i=1\hdots 6$ between surface points are non-uniform.}
    \label{fig:distance_discretization}
\end{figure}

\begin{figure}
    \begin{center}
    \includegraphics[width=1.\columnwidth]{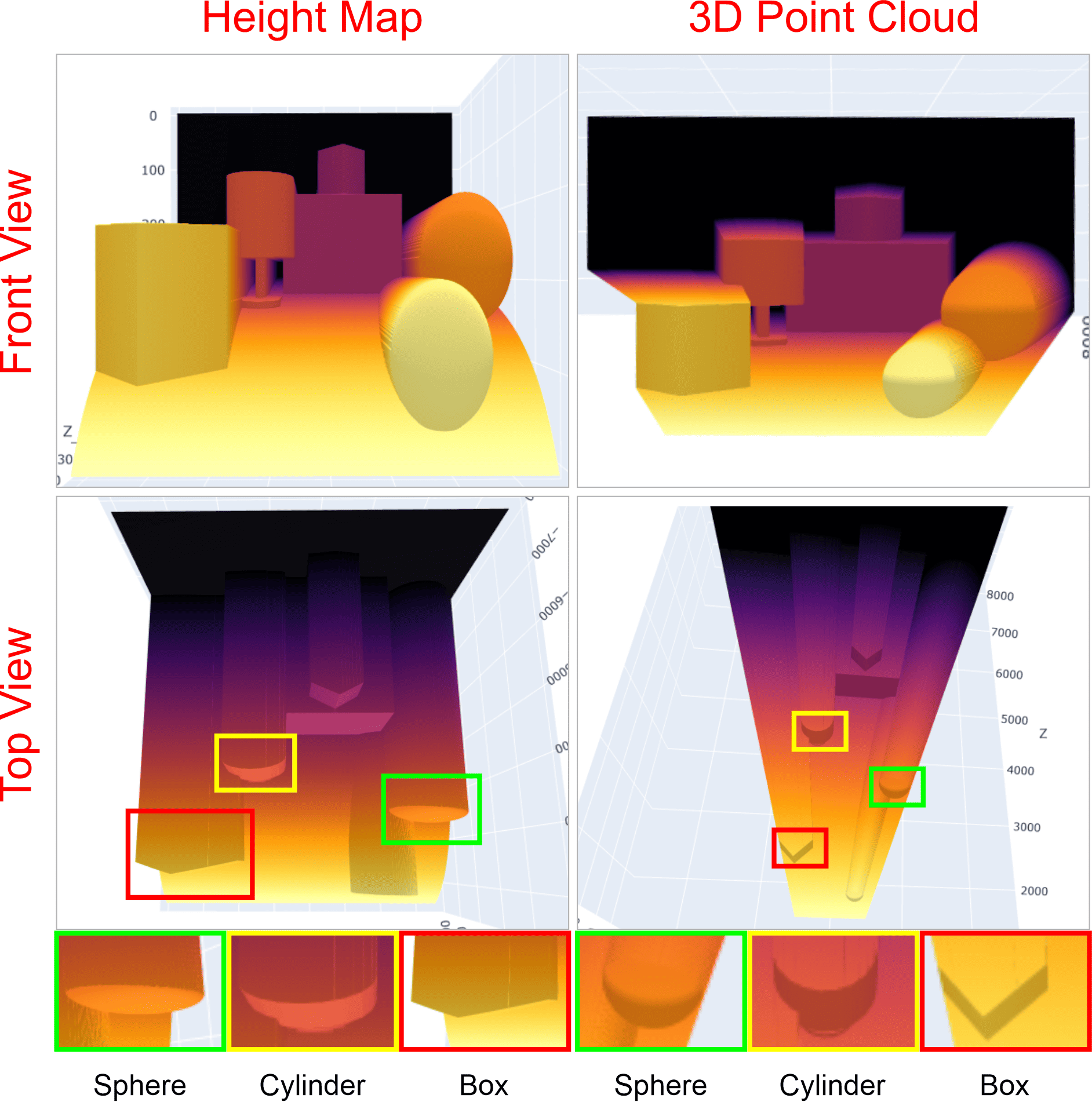}
    \end{center}
        \caption{This figure visually represents the $\Delta x$ uniform distances on the Height Map (left) and the $\Delta s$ uniform distances on the 3D Point Cloud. Observe that not dealing with the correct GC formula may lead to a wrong curvature analysis.}
    \label{fig:height_map_x_3D_point_cloud}
\end{figure}

\subsection{3D Scenes for Curvature Analysis}
\label{sec:synthetic_dataset}

In order to evaluate algorithms that estimate GC on shapes we created five 3D synthetic scenes\footnote{3D Synthetic Scenes Page: https://github.com/SherlonAlmeida/Stereo-Cameras-Simulation} composed of developable surfaces (zero GC), such as cylinders, boxes, and planes, and spheres representing a convex surface with constant positive curvature. Despite the simplicity of the scenes, it provides complete control over stereo image acquisition and precise ground truth for GC analysis.

The simulated environment was created using Unity 3D. We created two physical cameras inspired on Middlebury settings, with \textbf{Sensor Size (H,W)} = (14.8, 22.2)mm, \textbf{Image Size (H,W)} = (2000, 3000)px, \textbf{Focal Length (mm, px)} = (35, 4729.73) and \textbf{Pixel Size} = 0.0074mm. The cameras were YZ-aligned, and X-shifted by 200 millimeters to guarantee the images are rectified. Thus, \textbf{Baseline} = 200mm.

\autoref{fig:synthetic_dataset_rgbd} illustrates the five synthetic scenes created. As part of our ongoing work, we plan to expand the number of 3D synthetic scenes by including additional objects with known GC. For a more detailed analysis of GC, particularly in the MainScene, please refer to \autoref{fig:synthetic_dataset_mainscene_sections}.

\begin{figure*}[!ht]
    \begin{center}
    \includegraphics[width=\textwidth]{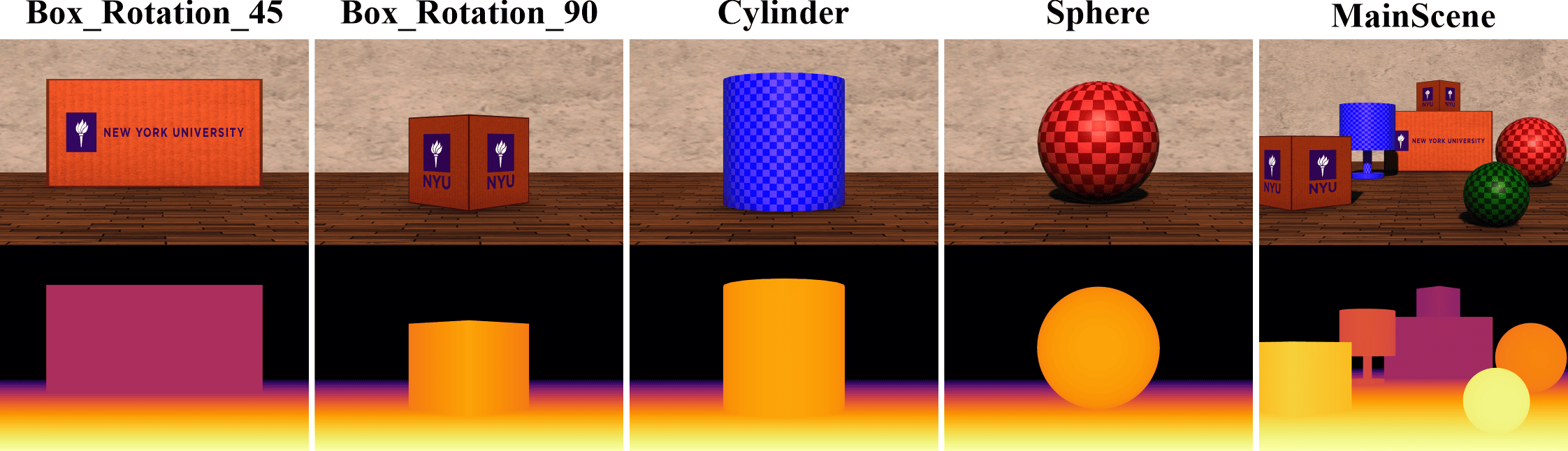}
    \end{center}
        \caption{\textbf{3D Synthetic Scenes:} Five scenes for depth estimation and curvature analysis, where: Box\_Rotation\_45, Box\_Rotation\_90, and Cylinder have $K = 0$; Sphere $K = \frac{1}{r^2}$; and MainScene mix all these objects in a scene. We provide the depth in meters (in this paper), left and right images with (2000, 3000)px resolution. The two spheres in the MainScene have the radius of $r=0.25$m and $r=0.125$m.
        }
    \label{fig:synthetic_dataset_rgbd}
\end{figure*}

\begin{figure*}[!ht]
    \begin{center}
    \includegraphics[width=\textwidth]{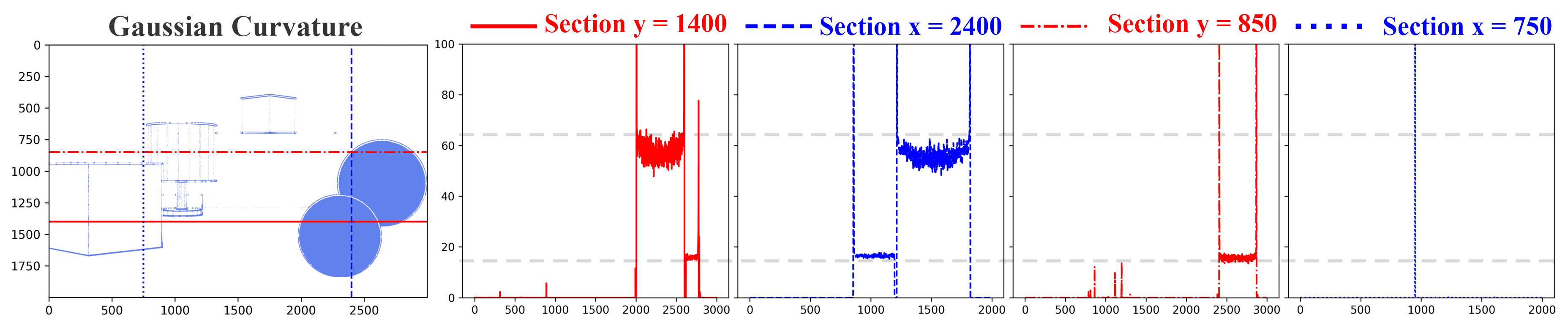}
    \end{center}
        \caption{GC Cross-sections on the MainScene data. Since second-order derivatives are sensitive to small changes,  we smooth the slices with a Gaussian filter of $\sigma = 2\, m $. The Y-section is plotted from left to right, and each X-section from top to bottom. Note (1) high curvature at the edges of 3D objects in the scene; (2) Approximate constant positive curvature inside the spheres, specifically 16 m$^{-2}$ for the red sphere, and 64 m$^{-2}$ for the green sphere, which are the correct GC values. (3) Zero curvature in the remaining areas.}
    \label{fig:synthetic_dataset_mainscene_sections}
\end{figure*}
\section{A Sparse Representation of 3D data}
\label{sec:sparse-representation}
We investigate a sparse representation associated with depth data by analyzing the Middlebury stereo dataset~\cite{scharstein2014high}, which provides accurate ground truth (GT) depth maps.

The Middlebury dataset includes a training partition, where GT data is publicly available.
The training dataset contains $15$ images, each of size approximately $2,000 \times 3,000$ pixels, yielding on the order of $9 \times 10^7= 15  \times 2,000 \times 3,000$ ground truth depth values. 

Middlebury results are provided as disparity maps, so we evaluated the GC of the 15 training images as follows:   we computed $Depth = \frac{f \cdot b}{d + doffs}$, where $f$ is the focal length in pixels, $b$ is the baseline in meters, and $d + doffs$ is the disparity value in pixels corrected by the center of projection offset. For each pixel, we compute the GC (throwing away boundary data) and aggregate the results and normalize them into a histogram, shown in~\autoref{fig:histogram-Gaussian-curvature-middleburry}.

\begin{figure}
    \centering
    \includegraphics[width=1.\columnwidth]{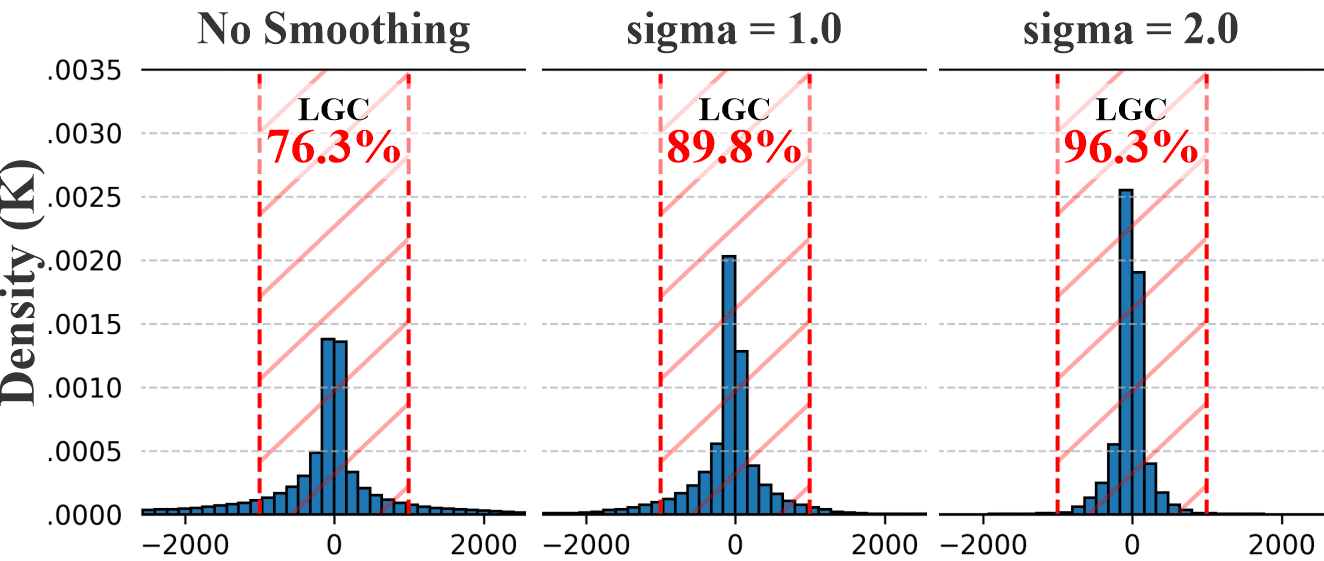}
    \vspace{0.35em}
    \captionof{figure}{Left: Normalized histogram of the GC, in units of inverse of square meters, across the 15 training images from the Middlebury dataset. Middle and Right: smoothing by $\sigma=1.0\, m, 2.0\, m$ respectively. We discarded the highest 20\% of $|K|$ values, so we discard depth boundary data to focus on items,  and the remaining $K$ values range between $[-32,370.4; 14,038.9]$. For visualization we plot $K$ values within $[-2,500; 2,500]m^{-2}$ in 30 bins uniformly distributed. Note that increased smoothing results in a higher LGC measure.}
    \label{fig:histogram-Gaussian-curvature-middleburry}
\end{figure}

Let us denote by \( h(K) \) the \textit{normalized histogram} of GC values computed empirically over the Middlebury dataset. This histogram can be interpreted as an \textit{empirical prior distribution} over GC. Based on this, we define a prior model over 3D surface geometry as
\begin{equation}
P(K) =  e^{-\mathcal{L}(K)},
\label{eq:histogram-density}
\end{equation}
where the loss function \( \mathcal{L}(K) \) is defined by
\begin{equation}
\mathcal{L}(K) = - \ln h(K).
\end{equation}
Our empirical findings suggest that GC is sparsely distributed in real-world 3D scenes, supporting our hypothesis that it captures structural regularities in surface geometry.

A practical and interpretable approximation to this loss is a  regularizer:
\begin{equation}
\mathcal{L}(K) = \alpha \, |K|^{\frac{1}{2}}= \alpha \, \sqrt{|\kappa_1 \, \kappa_2|},
\label{eq:loss-function-Gaussian-curvature}
\end{equation}
where \( \alpha > 0 \) is a weighting parameter in units of distance that helps match the empirical distribution. This loss function is the sparse $L^0$ norm of quantities $|\kappa_1|$ and $ |\kappa_2|$, i.e., 
\begin{align}
 \sqrt{|\kappa_1 \, \kappa_2|} = \lim_{p\rightarrow 0} 
\left( \frac{1}{2} (|\kappa_1|^p + |\kappa_2|^p)\right)^{\frac{1}{p}} \, ,
\label{eq:GC-L0}
\end{align}
(see \cite{Hardy52}. For completion we placed a proof of it in the supplementary material).
Thus equation~\eqref{eq:loss-function-Gaussian-curvature} is indeed the sparse loss function that aim to have as few components as possible of the principal curvatures. 

A natural metric to evaluate the sparsity of the GC distribution in an image or dataset is the Shannon entropy of \( h(K) \). Alternatively, we propose a simpler to compute and interpretable metric, which we call \textit{Low Gaussian Curvature (LGC)} metric, where a low entropy should map to a high LGC. This metric is defined as the percentage of all computed GC values that lie within a specified range, \( [-W, W] \). To avoid the impact of high values near surface boundaries (depth discontinuities), we remove the top 20\% of the highest \( |K| \) values from the entire analysis. As shown in \autoref{fig:histogram-Gaussian-curvature-middleburry}, 76.3\% of the GC values in the Middlebury dataset are concentrated in the range of \( W = 1{,}000 \, \mathrm{m}^{-2} \). After applying smoothing with $\sigma=2\, m$ to the data, the LGC increases to 96.3\% of values that fall within \( [-1{,}000; 1{,}000] \, \mathrm{m}^{-2} \), while the maximum absolute curvature reaches \( |K|_{\text{max}} = 32{,}370.4 \, \mathrm{m}^{-2} \).
\section{An analysis on depth reconstruction}
\label{sec:GC-metric}




We design experiments to evaluate to what extent depth-reconstruction algorithms incorporate  a sparse representation for the magnitude of the GC. These experiments are carried out on indoors scenes from both real-world data, using the Middlebury dataset
and controlled synthetic environments from our created 3D synthetic scenes.

\subsection{Middlebury Dataset} \label{subsec:middlebury_dataset}

The dataset includes a training partition, where GT data is publicly available, and a test partition, where only some calibration information and the left-right image pairs are provided.
In both partitions, algorithm rankings are available; however, only in the training set can the disparity map results of submitted methods be downloaded for further analysis. We obtained the full-resolution results of SOTA techniques\footnote{There are other relevant methods reported in the literature; however, our selection was based on top-ranked methods on the Middlebury, KITTI, and ETH3D benchmarks and the availability of source code.}, organized into two groups: \textbf{Group A)} FoundationStereo~\cite{wen2025foundationstereo}, LG-Stereo\footnote{Papers and code for LG-Stereo, BLMT-Stereo, and DepthFocus are not available yet.}, DEFOM-Stereo~\cite{jiang2025defom}, MonoStereo~\cite{cheng2025monster}, MonSter++~\cite{cheng2025monster++}, and AIO-Stereo~\cite{zhou2025all} which represent the new generation of stereo approaches that combine stereo matching with monocular features; and \textbf{Group B)} RAFT-Stereo~\cite{lipson2021raft}, CREStereo~\cite{li2022practical}, DLNR~\cite{zhao2023high}, Selective-IGEV~\cite{wang2024selective}, and S2M2~\cite{min2025s2m2}, corresponding to previous SOTA techniques. We also evaluated the results of BLMT-Stereo\footnotemark[\value{footnote}] and DepthFocus\footnotemark[\value{footnote}], two new approaches currently under review for CVPR 2026. Although the papers and source codes were not available at the time of our submission, we were able to analyze their results using the GC evaluation.

The experiment on~\autoref{fig:middlebury_dataset_avgerr_x_curvature_analysis} compares the disparity AvgError performance and the median $|K|$ values of GC for all techniques. Note that the new SOTA methods (Group A) are the best-performing approaches in the Middlebury ranking and tend to estimate lower curvature magnitudes compared to the Group B methods. \autoref{tab:eval_metrics} presents a summary of the evaluation metrics.

\begin{table}
    \centering
    \resizebox{\columnwidth}{!}{
    \begin{tabular}{cc}
        \toprule
        \textbf{Metric} & \textbf{Definition} \\
        \midrule
        \makecell{Average  Absolute Error \\ (AvgError)} &
        $\displaystyle \frac{1}{MN}\sum_{i=1}^{M}\sum_{j=1}^{N} |d_{ij} - \hat{d}_{ij}|$ \\
        \midrule
        
        \makecell{Root Mean Squared Error \\ (RMS)} &
        $\displaystyle 
        \sqrt{\frac{1}{MN}\sum_{i=1}^{M}\sum_{j=1}^{N} (d_{ij} - \hat{d}_{ij})^2}$ \\
        \midrule
        
        \makecell{Bad-N \\ (\% of pixels with error $>N$)} &
        $\displaystyle 100 \times 
        \frac{\big| \{\, (i,j) \;|\; |d_{ij} - \hat{d}_{ij}| > N \,\} \big|}
        {\big| \{\, (i,j) \;\text{ valid pixels } \} \big|}$ \\
        \midrule
        
        \makecell{Average Absolute \\ Gaussian Curvature (Avg$|K|$)} &
        $\displaystyle \frac{1}{MN}\sum_{i=1}^{M}\sum_{j=1}^{N} |K_{ij}|$ \\
        \midrule
        
        \makecell{Normal Error \\ (NormalsErr)} &
        \makecell{$\displaystyle
        \frac{1}{|\Omega|}
        \sum_{(i,j)\in \Omega}
        \left(
        1 -
        \frac{
        \mathbf{n}_{ij}^{\text{gt}} \cdot \mathbf{n}_{ij}^{\text{pred}}
        }{
        \lVert \mathbf{n}_{ij}^{\text{gt}} \rVert\,
        \lVert \mathbf{n}_{ij}^{\text{pred}} \rVert
        }
        \right)
        $\\where $\Omega$ is the set of valid pixels.}\\
        \midrule
        
        \makecell{Low Gaussian \\ Curvature (LGC)} &
        $\displaystyle
        100 \times
        \frac{
        \left|\{(i,j)\mid -W \le K_{ij} \le W\}\right|
        }{
        \left|\{\,(i,j)\;\text{valid pixels}\,\}\right|
        }$ \\
        \bottomrule
    \end{tabular}
    }
    \caption{Summary of evaluation metrics adopted in this work, covering disparity and depth errors, curvature measurements, and surface-normal consistency.}
    \label{tab:eval_metrics}
\end{table}

\begin{figure}
    \begin{center}
    \includegraphics[width=\columnwidth]{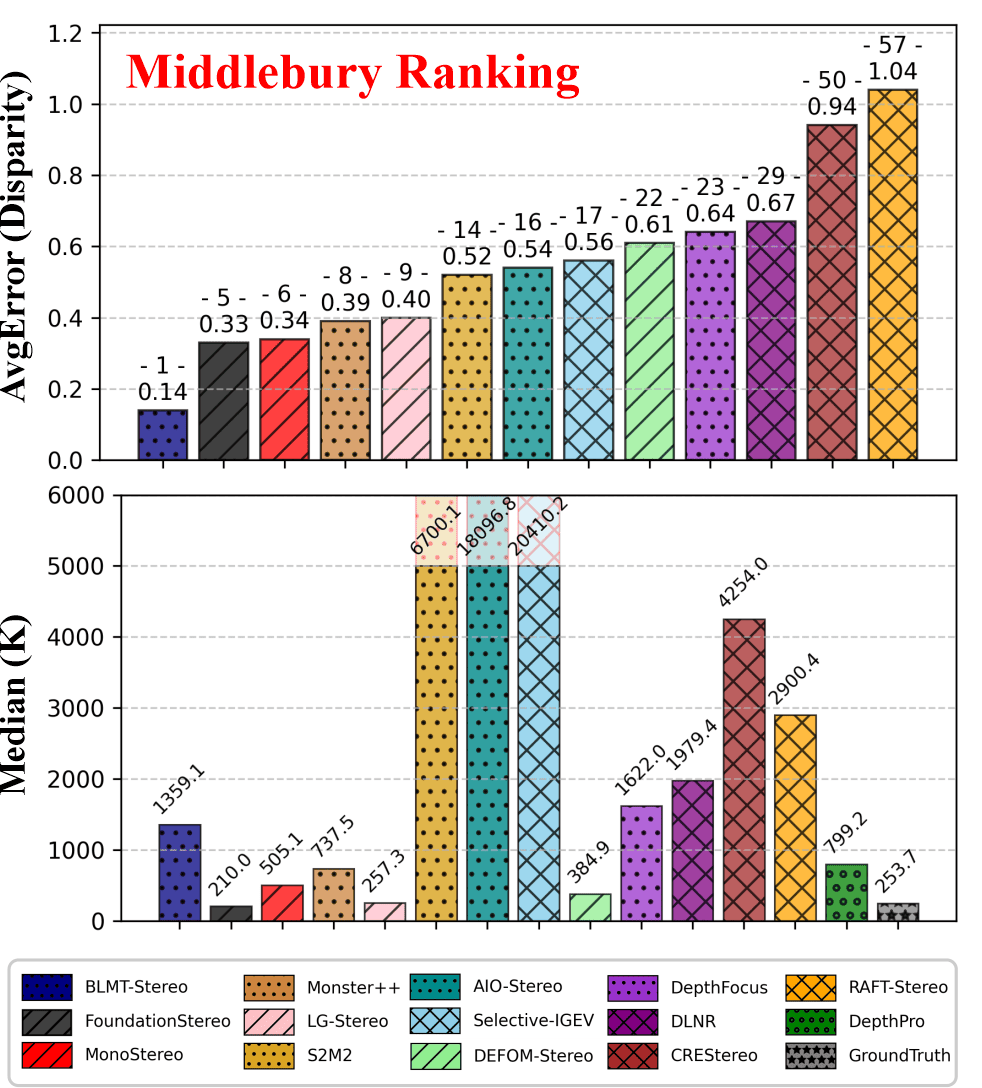}
    \end{center}
        \caption{\textbf{Middlebury Dataset - Curvature x AvgError Analysis:} Top: shows the overall position of each approach in the Middlebury ranking, along with its average disparity error (AvgError). Bottom: presents the median of absolute GC for all techniques and the GT.}
        \label{fig:middlebury_dataset_avgerr_x_curvature_analysis}
\end{figure}

A more detailed experiment can be seen on~\autoref{fig:middlebury_dataset_curvature_distribution} and~\autoref{tab:middlebury_benchmark}, where we analyze the GC distribution for each technique. Note that FoundationStereo has 81.3\% of its GC between $[-1{,}000, 1{,}000]m^{-2}$ $K$ values. Additionally, the same trend of minimizing GC across the 15 Middlebury training set was observed in the \colorbox{top5}{Top 5} best techniques, all of them from Group A\footnote{We invite the reader to refer to the Supplementary Material for a more comprehensive study.}.
Despite achieving low AvgError, techniques from Group B exhibit also lower LGC, with a tendency toward higher $|K|$ values compared to those observed in the \colorbox{top5}{Top 5} Group A techniques.

\begin{figure}
    \begin{center}
    \includegraphics[width=1.0\columnwidth]{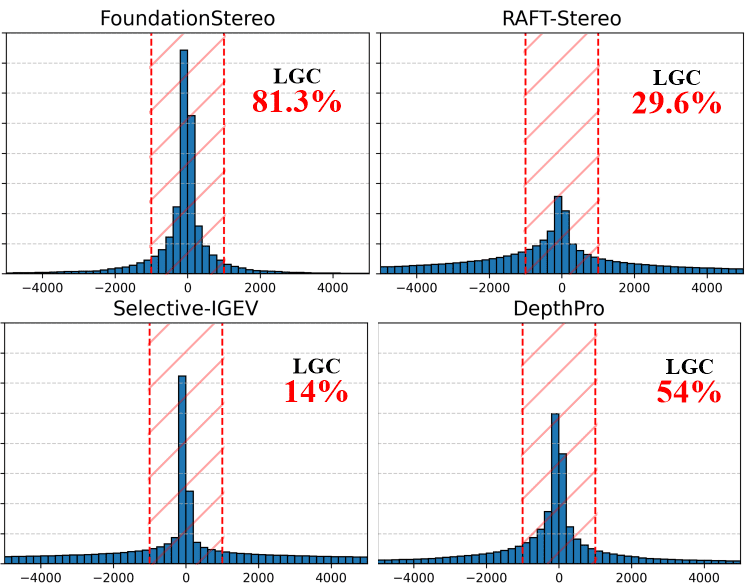}
    \end{center}
        \caption{\textbf{Curvature Distribution:} each plot presents a normalized histogram of the GC distribution for all 15 training images from the Middlebury dataset. We discarded the highest 20\% of $|K|$ values to avoid depth discontinuities, and plotted the remaining $K$ values within $[-4,000, 4,000]m^{-2}$ in 50 bins uniformly distributed.
    }
    \label{fig:middlebury_dataset_curvature_distribution}
\end{figure}

\begin{table*}[t]
\centering
\footnotesize
\setlength{\tabcolsep}{24pt} 
\begin{tabular}{lccccc}
\toprule
Technique        & LGC  & AvgError $\downarrow$ & RMS $\downarrow$ & Bad 2.0 $\downarrow$ & Bad 4.0 $\downarrow$ \\
\midrule
FoundationStereo~\cite{wen2025foundationstereo} & 81.3 & \cellcolor{top3}$0.33^{2}$  & \cellcolor{top5}$2.86^{4}$  & \cellcolor{top3}$0.79^{2}$  & \cellcolor{top3}$0.40^{3}$ \\
LG-Stereo                                       & 78.1 & \cellcolor{top5}$0.40^{5}$  & \cellcolor{top5}$2.94^{5}$  & \cellcolor{top5}$1.04^{5}$  & \cellcolor{top5}$0.44^{4}$ \\
DEFOM-Stereo~\cite{jiang2025defom}              & 69.2 & $0.61^{9}$  & $3.66^{8}$  & $2.26^{10}$ & $1.19^{10}$ \\
MonoStereo~\cite{cheng2025monster}              & 64.6 & \cellcolor{top3}$0.34^{3}$  & \cellcolor{top3}$2.46^{2}$  & \cellcolor{top3}$0.94^{3}$  & \cellcolor{top3}$0.35^{2}$ \\
Monster++~\cite{cheng2025monster++}             & 56.5 & \cellcolor{top5}$0.39^{4}$  & \cellcolor{top3}$2.54^{3}$  & $1.17^{6}$  & \cellcolor{top5}$0.47^{5}$ \\
BLMT-Stereo                                     & 43.1 & \cellcolor{top1}$0.14^{1}$  & \cellcolor{top1}$1.33^{1}$  & \cellcolor{top1}$0.17^{1}$  & \cellcolor{top1}$0.09^{1}$ \\
DepthFocus                                      & 39.9 & $0.64^{10}$ & $4.27^{11}$ & $1.84^{7}$  & $1.08^{9}$ \\
DLNR~\cite{zhao2023high}                        & 36.3 & $0.67^{11}$ & $3.90^{10}$ & $2.92^{11}$ & $1.41^{11}$ \\
RAFT-Stereo~\cite{lipson2021raft}               & 29.6 & $1.04^{13}$ & $5.25^{13}$ & $5.25^{13}$ & $2.89^{13}$ \\
CREStereo~\cite{li2022practical}                & 25.8 & $0.94^{12}$ & $5.21^{12}$ & $4.01^{12}$ & $2.04^{12}$ \\
S2M2~\cite{min2025s2m2}                         & 17.5 & $0.52^{6}$  & $3.73^{9}$  & \cellcolor{top5}$1.00^{4}$  & $0.59^{6}$ \\
AIO-Stereo~\cite{zhou2025all}                   & 15.3 & $0.54^{7}$  & $3.50^{6}$  & $1.97^{8}$  & $0.82^{7}$ \\
Selective-IGEV~\cite{wang2024selective}         & 14.0 & $0.56^{8}$  & $3.57^{7}$  & $2.18^{9}$  & $0.92^{8}$ \\

\midrule
Ground Truth     & 76.3 & ---    & --- & ---     & ---     \\
\bottomrule
\end{tabular}
\caption{This table presents the Middlebury benchmark ranking for the 15 training images, with techniques listed in descending LGC order. Superscripts indicate each method’s rank among the compared techniques for the metrics AvgError, RMS, Bad 2.0, and Bad 4.0. Darker cell shading highlights better performance, indicating the technique is among the \colorbox{top1}{Top 1}, \colorbox{top3}{Top 3}, or \colorbox{top5}{Top 5} best approaches. Notably, top-performing methods (Group A) generally exhibit higher LGC (i.e., lower GC).}
\label{tab:middlebury_benchmark}
\end{table*}

Considering that humans easily perceive the shape of flat objects (\eg walls, ground, doors, etc), in the following experiment we present a visual analysis of curvature, using the ``Piano'' data from Middlebury dataset. We compared the $K$ values of GC from the Middlebury's provided GT and the FoundationStereo reconstruction, which is the best (lowest AvgError) on the Middlebury ranking for this data with public code. We show the left image, and depth for GT and FoundationStereo in~\autoref{fig:middlebury_dataset_piano_depth}, emphasizing the piano bench and sheet music, that we known are developable surfaces with $K = 0$. Despite good depth estimation, Middlebury's GT provides a binary mask with measurement problems for each disparity map. Besides filtering the highest $|K|$ values of the GT, previously we removed these NaN values from the GC analysis for GT. An example of these problematic data can be seen as black positions on the GT depth in~\autoref{fig:middlebury_dataset_piano_depth}.

\begin{figure}
    \begin{center}
    \includegraphics[width=\columnwidth]{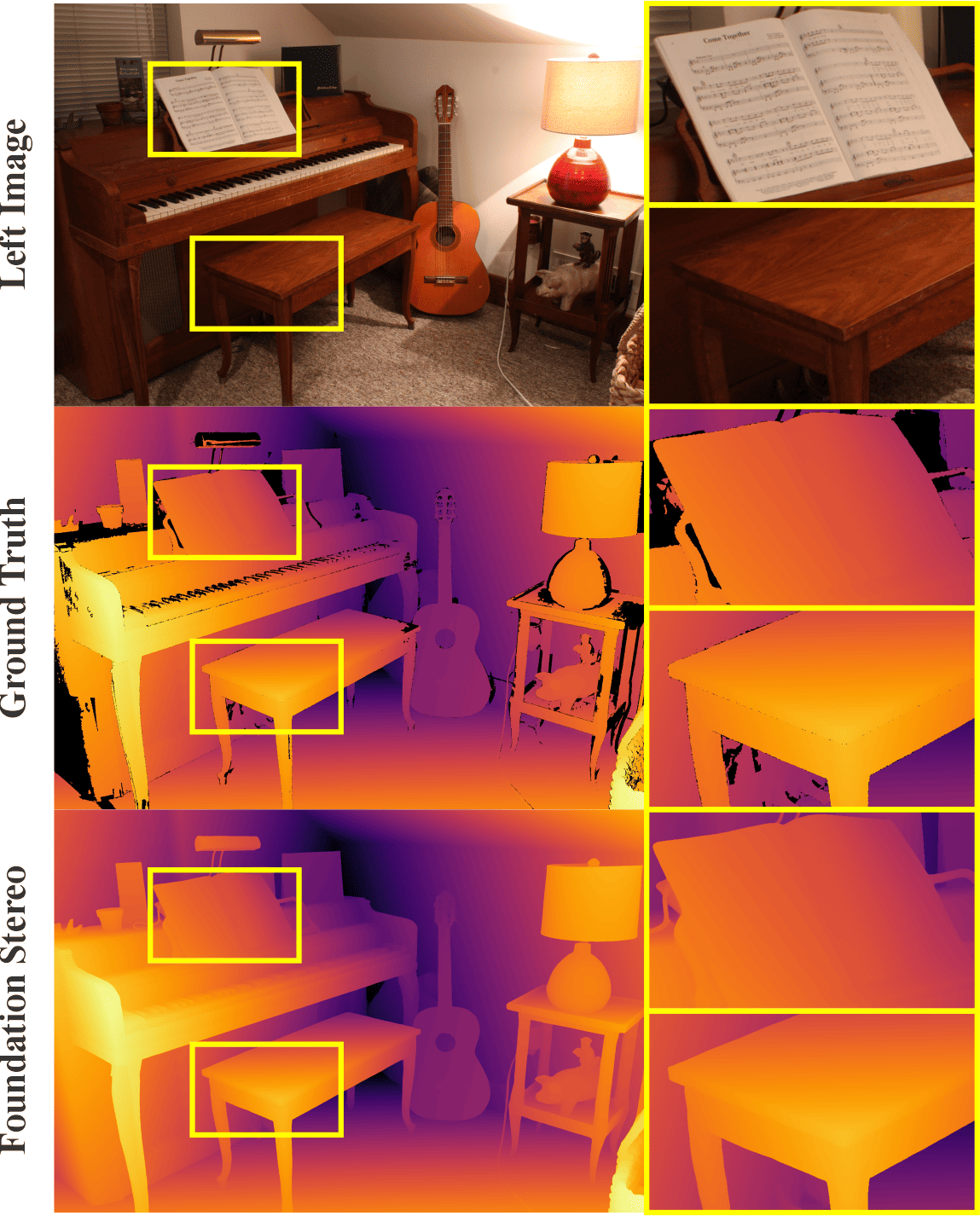}
    \end{center}
    \caption{RGB left image x Depth.}
    \label{fig:middlebury_dataset_piano_depth}
\end{figure}

In~\autoref{fig:middlebury_dataset_piano} the GT, FoundationStereo, and Selective-IGEV present a noisy representation of GC (first row), then we also present $K$ values smoothed ($\sigma = 2.0\, m$). Also, the non-black coordinates represent $|K|$ values lower than 1,000$m^{-2}$.

Observe that edges throughout the image exhibit abrupt changes in curvature, while flat surfaces do not. Another interesting observation is that the non-smoothed reconstruction from FoundationStereo produces $|K| > 1{,}000m^{-2}$ in some areas of the lampshade — a cylindrical surface that should have a constant $K = 0$ — as well as higher $K$ values on the music notes printed on the sheet music, which, as a flat piece of paper, should also have $K = 0$. On the other hand, the bottom part of the lampshade shows an approximation of positive (red) $K$ values, which is expected given the spherical shape's positive principal curvatures. Looking at the $K$ values estimated by Selective-IGEV, it is evident that the output is noisier than those of other approaches. While FoundationStereo exhibited low GC on flat surfaces even in the original output, Selective-IGEV only revealed clearer structures after smoothing.

In \autoref{tab:middlebury_benchmark_normals}, we extend our analysis to the normal error between the reconstructed point clouds from each technique and the GT across all 15 Middlebury training images. Notably, the top-performing (\colorbox{top5}{Top 5}) methods in terms of Normal Average Error (NormAvg) belong to Group A. Moreover, Selective-IGEV showed poor alignment of surface normals with the expected GT across all 15 images.

In Figures~\ref{fig:normals_piano} and~\ref{fig:normals_playtable}, we qualitatively present the normal reconstructions for the GT, BLMT-Stereo, Foundation-Stereo, and Selective-IGEV on the Piano and Playtable data, respectively. BLMT-Stereo and Foundation-Stereo produce depth maps with more coherent surface normals compared to Selective-IGEV. While BLMT-Stereo aligns normals more accurately in thin structures and near depth discontinuities, Foundation-Stereo achieves better consistency in regions with low absolute Gaussian curvature. In both cases, Selective-IGEV struggles, producing surfaces with high absolute GC and poorer normal alignment with the GT.

\begin{table*}[t]
\centering
\resizebox{\textwidth}{!}{
\begin{tabular}{lcccccccccccccccc}
\toprule
Technique & NormAvg $\downarrow$ & Adirondack & ArtL & Jadeplant & Motorcycle & MotorcycleE & Piano & PianoL & Pipes & Playroom & Playtable & PlaytableP & Recycle & Shelves & Teddy & Vintage \\
\midrule

BLMT-Stereo         & 6.44      & \cellcolor{top1}$2.09^{1}$    & \cellcolor{top5}$18.08^{4}$   & \cellcolor{top3}$4.90^{3}$    & \cellcolor{top1}$6.31^{1}$    & \cellcolor{top1}$6.30^{1}$    & \cellcolor{top1}$2.93^{1}$    & \cellcolor{top1}$3.01^{1}$    & \cellcolor{top1}$8.36^{1}$    & \cellcolor{top1}$3.81^{1}$    & \cellcolor{top1}$2.60^{1}$    & \cellcolor{top1}$2.98^{1}$    & \cellcolor{top1}$1.86^{1}$    & \cellcolor{top1}$4.40^{1}$    & \cellcolor{top3}$25.92^{2}$   & \cellcolor{top1}$3.08^{1}$ \\
FoundationStereo    & 6.75      & \cellcolor{top3}$2.24^{2}$    & \cellcolor{top1}$15.92^{1}$   & \cellcolor{top3}$4.84^{2}$    & \cellcolor{top5}$7.88^{4}$    & \cellcolor{top3}$7.84^{3}$    & \cellcolor{top3}$3.70^{3}$    & \cellcolor{top3}$3.97^{3}$    & \cellcolor{top3}$9.04^{3}$    & \cellcolor{top3}$4.91^{3}$    & \cellcolor{top3}$3.09^{3}$    & \cellcolor{top5}$3.52^{4}$    & \cellcolor{top5}$2.74^{5}$    & \cellcolor{top3}$5.03^{3}$    & \cellcolor{top1}$23.23^{1}$   & \cellcolor{top3}$3.37^{3}$ \\
LG-Stereo           & 6.78      & \cellcolor{top1}$2.09^{1}$    & \cellcolor{top3}$17.97^{3}$   & \cellcolor{top5}$4.91^{4}$    & \cellcolor{top3}$7.00^{2}$    & \cellcolor{top3}$7.05^{2}$    & \cellcolor{top3}$3.36^{2}$    & \cellcolor{top3}$3.68^{2}$    & \cellcolor{top3}$8.46^{2}$    & \cellcolor{top3}$4.81^{2}$    & \cellcolor{top3}$2.92^{2}$    & \cellcolor{top3}$3.24^{2}$    & \cellcolor{top3}$2.05^{2}$    & \cellcolor{top3}$4.90^{2}$    & \cellcolor{top3}$25.98^{3}$   & \cellcolor{top3}$3.24^{2}$ \\
MonoStereo          & 7.29      & \cellcolor{top3}$2.32^{3}$    & \cellcolor{top3}$17.75^{2}$   & \cellcolor{top1}$4.60^{1}$    & \cellcolor{top3}$7.86^{3}$    & \cellcolor{top5}$8.32^{4}$    & \cellcolor{top5}$3.88^{4}$    & \cellcolor{top5}$4.63^{4}$    & \cellcolor{top5}$9.92^{4}$    & \cellcolor{top5}$5.24^{4}$    & \cellcolor{top5}$3.14^{4}$    & \cellcolor{top3}$3.36^{3}$    & \cellcolor{top3}$2.36^{3}$    & \cellcolor{top5}$6.33^{4}$    & \cellcolor{top5}$26.02^{4}$   & \cellcolor{top5}$3.64^{4}$ \\
Monster++           & 8.13      & \cellcolor{top5}$2.70^{4}$    & \cellcolor{top5}$18.26^{5}$   & \cellcolor{top5}$5.05^{5}$    & 8.62          & \cellcolor{top5}$8.38^{5}$    & 4.81          & \cellcolor{top5}$5.86^{5}$    & 11.43         & \cellcolor{top5}$6.04^{5}$    & 4.05          & \cellcolor{top5}$4.00^{5}$    & 2.76          & 9.57          & \cellcolor{top5}$26.26^{5}$   & 4.19 \\
DEFOM-Stereo        & 8.18      & \cellcolor{top5}$2.77^{5}$    & 19.18         & 6.82          & \cellcolor{top5}$8.56^{5}$    & 8.70          & 4.71          & 6.18          & \cellcolor{top5}$11.31^{5}$   & 6.53          & \cellcolor{top5}$3.90^{5}$    & 4.15          & \cellcolor{top5}$2.68^{4}$    & \cellcolor{top5}$7.16^{5}$    & 26.43         & \cellcolor{top5}$3.66^{5}$ \\
DepthFocus          & 8.93      & 4.24          & 19.38         & 10.49         & 8.79          & 8.83          & \cellcolor{top5}$4.61^{5}$    & 7.85          & 12.05         & 6.97          & 4.01          & 4.17          & 3.43          & 7.35          & 27.34         & 4.39 \\
DLNR                & 9.05      & 3.50          & 19.76         & 7.20          & 9.40          & 9.50          & 6.22          & 8.88          & 11.60         & 6.92          & 4.20          & 4.47          & 3.36          & 8.68          & 26.65         & 5.42 \\
S2M2                & 9.82      & 5.51          & 19.50         & 6.40          & 10.63         & 10.59         & 5.99          & 6.07          & 15.46         & 8.02          & 5.82          & 5.94          & 4.58          & 8.54          & 28.00         & 6.20 \\
RAFT-Stereo         & 11.11     & 4.94          & 21.57         & 9.57          & 11.44         & 11.55         & 7.77          & 11.84         & 14.67         & 8.45          & 5.90          & 6.16          & 4.38          & 12.06         & 27.08         & 9.33 \\
CREStereo           & 11.61     & 5.69          & 20.46         & 10.30         & 12.55         & 12.52         & 9.73          & 12.94         & 15.87         & 9.78          & 7.80          & 7.84          & 5.19          & 10.17         & 26.95         & 6.43 \\
AIO-Stereo          & 16.65     & 10.03         & 22.05         & 12.28         & 15.49         & 15.55         & 14.60         & 16.45         & 23.45         & 13.21         & 14.20         & 14.81         & 10.46         & 17.65         & 28.48         & 20.99 \\
Selective-IGEV      & 17.90     & 10.80         & 22.62         & 13.61         & 17.23         & 17.28         & 16.88         & 18.70         & 24.09         & 14.31         & 15.87         & 16.44         & 11.57         & 18.15         & 28.98         & 22.00 \\

\midrule
DepthPro & 19.86 & 10.74 & 42.77 & 31.17 & 23.37 & 23.37 & 11.89 & 11.89 & 31.91 & 18.42 & 8.87 & 9.17 & 9.62 & 14.88 & 42.35 & 7.46 \\

\bottomrule
\end{tabular}
} 
\caption{This table presents techniques listed in ascending Normal Average Error (NormAvg) order for the 15 Middlebury training images. Superscripts indicate the lowest 1-5th Normal Error (NormalsErr) per technique for each one of the 15 Middlebury images. Darker cell shading highlights lower NormalsErr, indicating the technique is among the \colorbox{top1}{Top 1}, \colorbox{top3}{Top 3}, or \colorbox{top5}{Top 5} approaches with the lowest NormalsErr. Notably, the top-performing methods are from the Group A, and predict surface normals that are closely aligned with the ground truth.}
\label{tab:middlebury_benchmark_normals}
\end{table*}

\begin{figure*}
    \begin{center}
    \includegraphics[width=\textwidth]{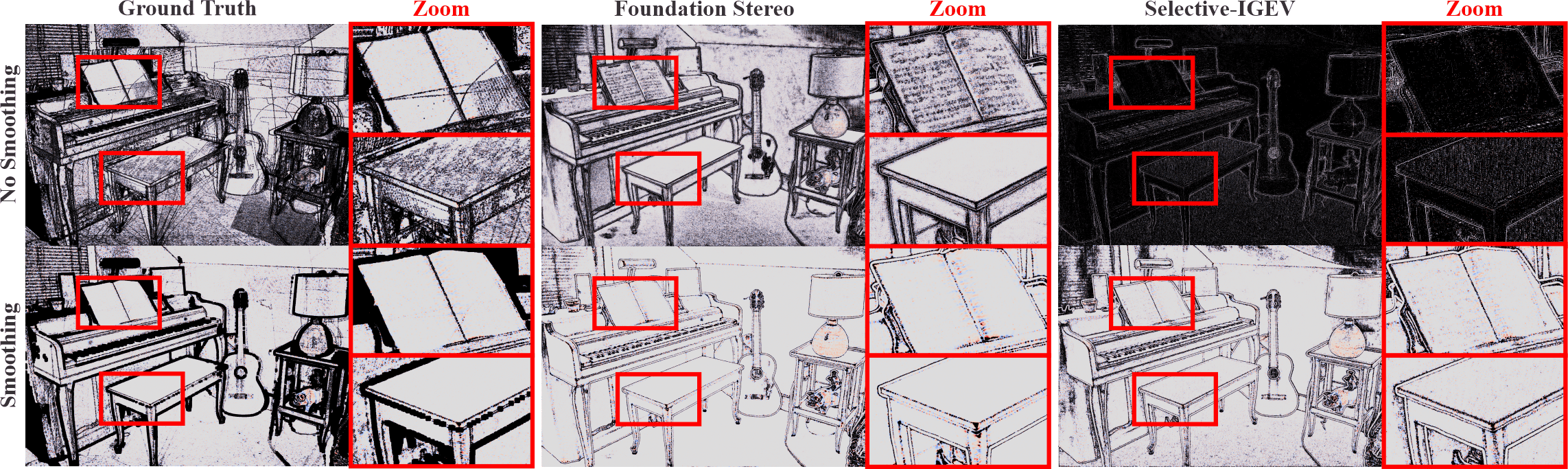}
    \end{center}
        \caption{\textbf{GT x SOTA approaches:} a point-wise analysis of curvature for "Piano" image. Black coordinates represent values of $|K| > 1{,}000m^{-2}$. For the GT, black coordinates also represent NaN values, which are measurement inconsistencies during Middlebury disparity estimation. In the second row, we applied smoothing with $\sigma = 2\, m$ in the 3D point cloud before computing the GC.
}
    \label{fig:middlebury_dataset_piano}
\end{figure*}

\begin{figure*}[!ht]
    \begin{center}
    \includegraphics[width=\textwidth]{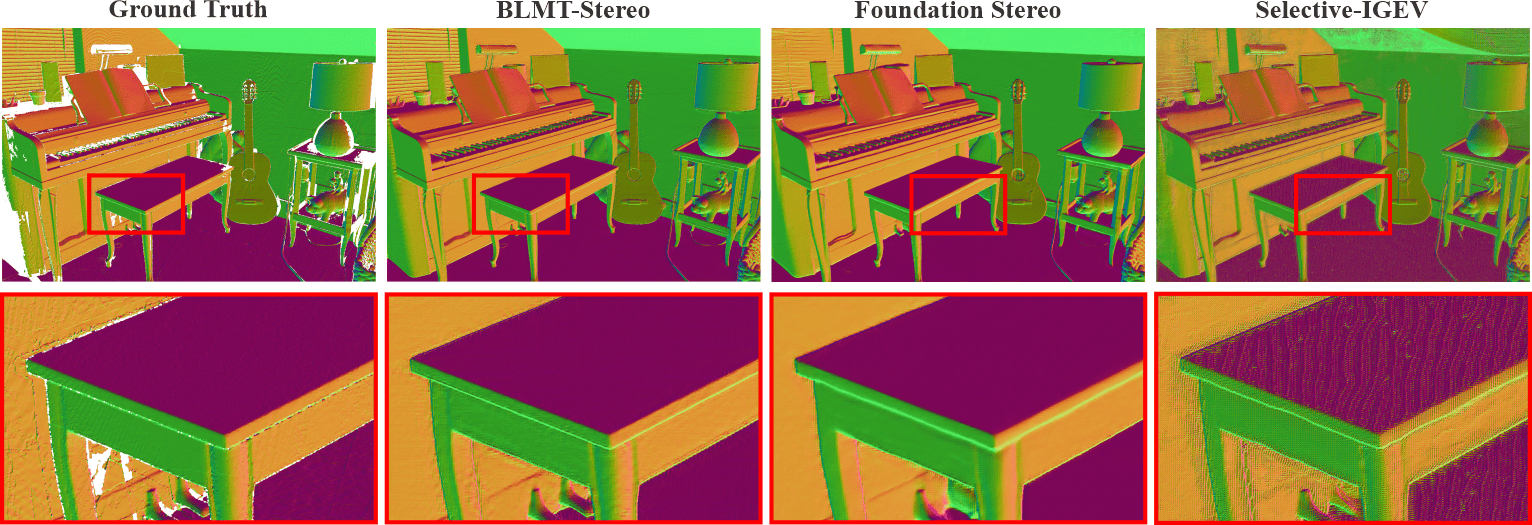}
    \end{center}
        \caption{\textbf{Normals Analysis:} Qualitative comparison on normals reconstruction for Piano data.
        }
    \label{fig:normals_piano}
\end{figure*}

\begin{figure*}[!ht]
    \begin{center}
    \includegraphics[width=\textwidth]{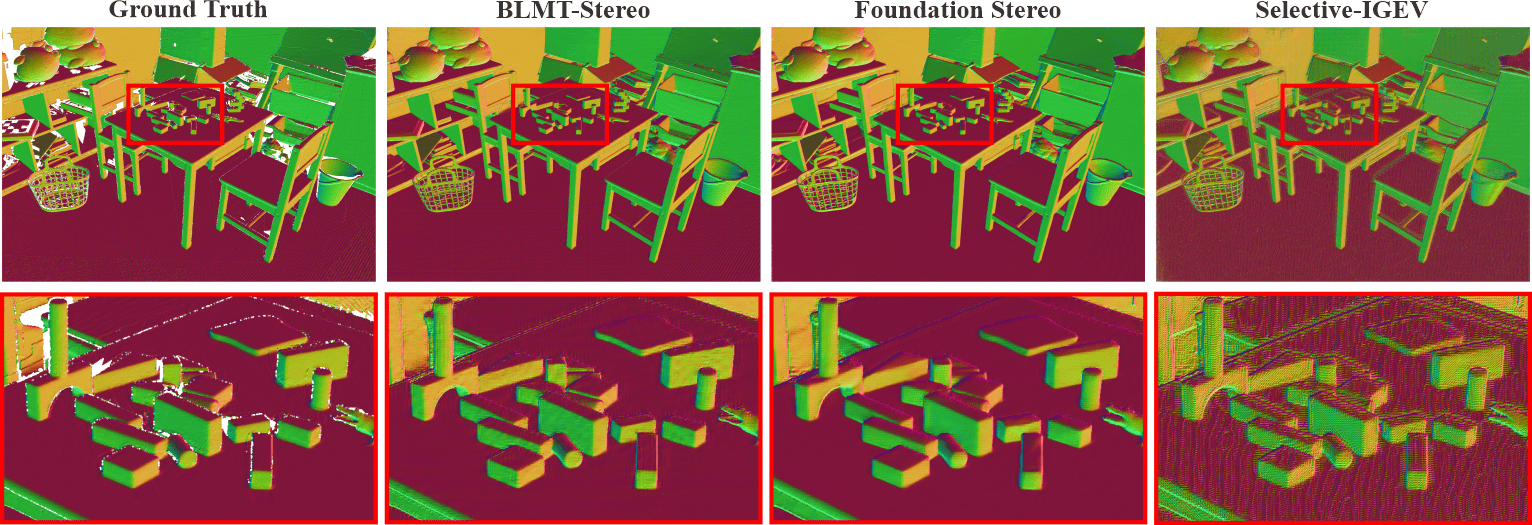}
    \end{center}
        \caption{\textbf{Normals Analysis:} Qualitative comparison on normals reconstruction for Playtable data.
        }
    \label{fig:normals_playtable}
\end{figure*}

\begin{figure}[!ht]
    \begin{center}
    \includegraphics[width=\columnwidth]{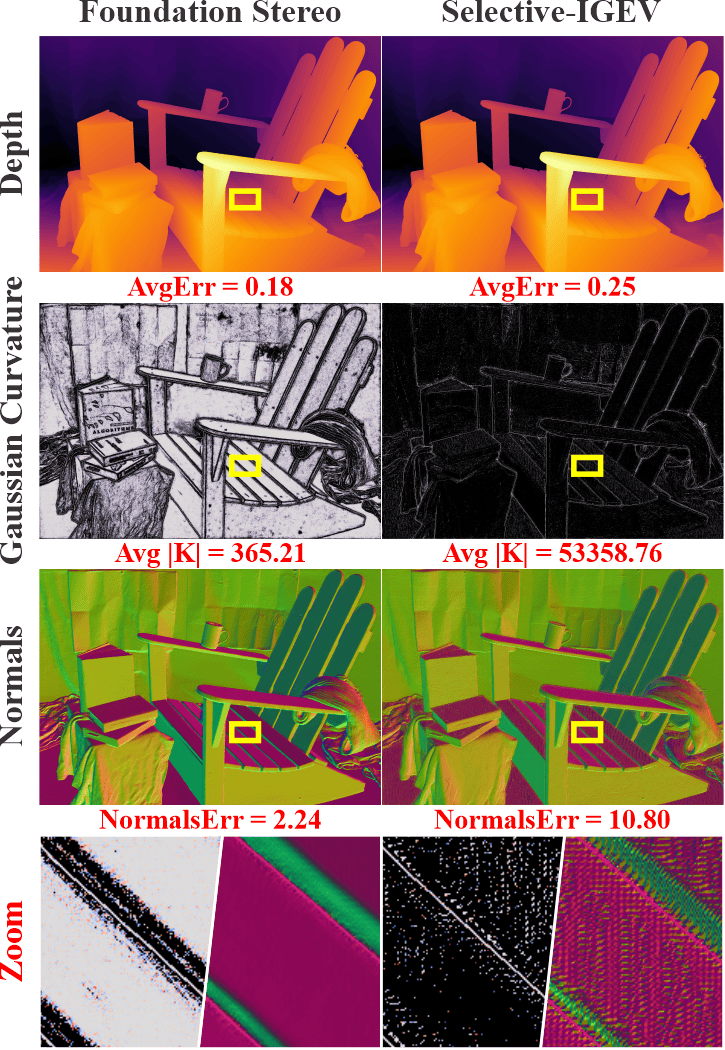}
    \end{center}
        \caption{\textbf{Quantitative vs. Qualitative comparison:} We present the average disparity error (AvgError), average absolute Gaussian curvature (Avg$|K|$), and normal error (NormalsErr) for the Foundation-Stereo and Selective-IGEV techniques on the Adirondack data. We also include a zoomed-in region to illustrate the pointwise consistency of Gaussian curvature and surface normals. Note that despite its low AvgError, the reconstruction produced by Selective-IGEV lacks smoothness.
        }
    \label{fig:quant_quali_comparison}
\end{figure}

In~\autoref{fig:quant_quali_comparison}, we also present additional qualitative and quantitative results for the Adirondack data. Although both Foundation-Stereo and Selective-IGEV achieve accurate depth reconstruction (low AvgError), their Gaussian curvature and normal consistency differ substantially. Foundation-Stereo exhibits lower Avg$|K|$ and lower NormalsErr, which may account for its overall better performance.

Our experiments demonstrate that the new SOTA approaches (Group A) are able to estimate more concise structures during disparity/depth reconstruction, which means the global structure of a surface is kept. Unlike new SOTA methods, Group B techniques do not preserve global surface structures, despite high accuracy in AvgError, which suggests that unifying stereo and monocular features is crucial to preserve 3D data relationships.

\subsection{3D Synthetic Scenes} \label{subsec:synthetic_dataset}

In our 3D synthetic scenes, the data were designed to make the GC analysis easier by using simple objects with well-known curvatures. We ensured that planar surfaces -- such as boxes, the ground, and walls -- have $K = 0$ in all scenes by generating the GT depth map through orthogonal projection of each surface point onto the left camera’s image plane, with a (2000, 3000)px resolution.

For the following experiments, we obtained the source code for: FoundationStereo, RAFT-Stereo, and Selective-IGEV, representing the aforementioned stereo approaches; and for DepthAnythingV2~\cite{yang2024depthv2} and DepthPro~\cite{bochkovskii2024depth}, representing the SOTA in Monocular Depth Estimation (MDE).

\autoref{fig:synthetic_dataset_avgerr_x_curvature_analysis} presents a quantitative analysis of GC and average depth error across the five scenes in the 3D synthetic scenes. The left-most bar in each graph represents the curvature of the GT, while the other bars show the curvature estimated by stereo and monocular methods. The right-most plot for each scene shows the average depth error (AvgError) in centimeters. Notably, FoundationStereo, which achieves the lowest average error -- less than one centimeter in all scenes -- also exhibits the lowest GC. See the depth and curvature results for each one of them in~\autoref{fig:sota_stereo_GC}.

\begin{figure}[!ht]
    \begin{center}
    \includegraphics[width=0.90\columnwidth]{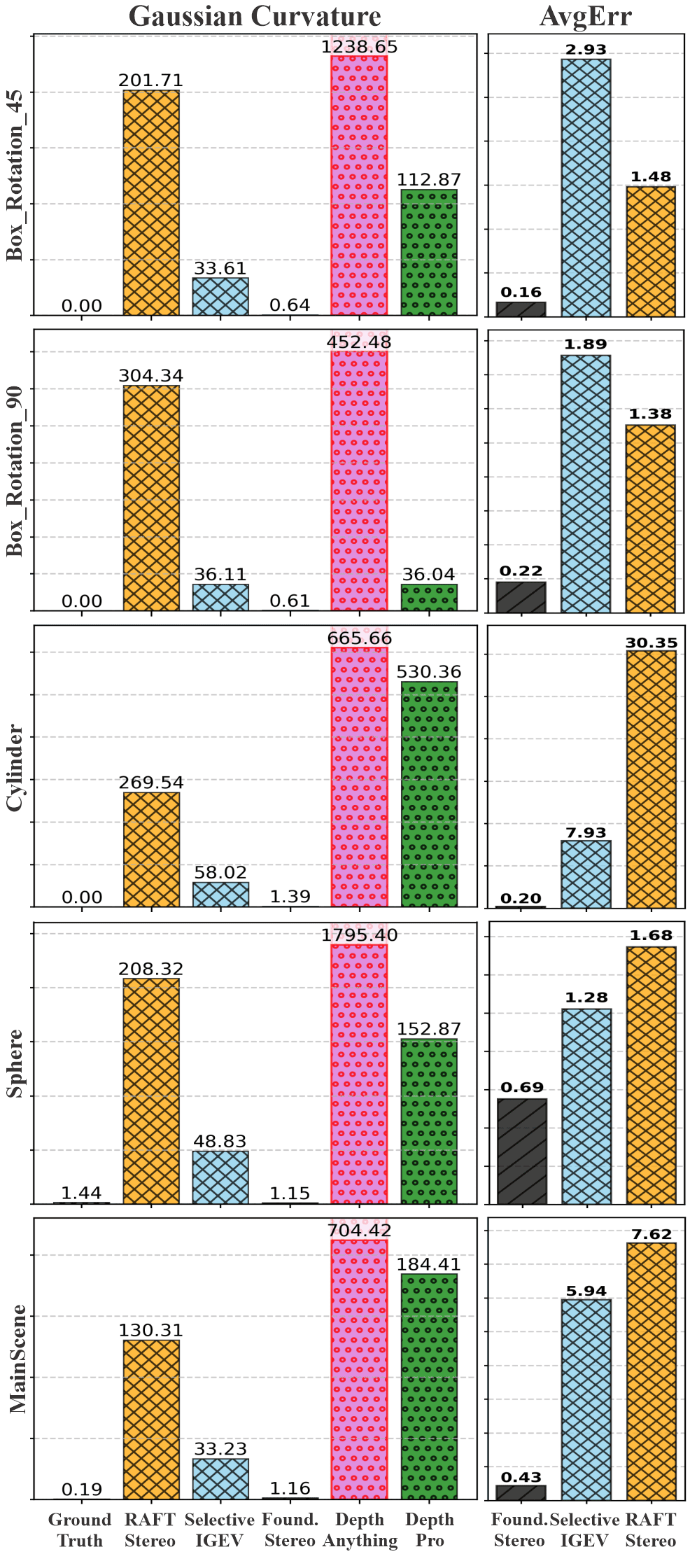}
    \end{center}
        \caption{This figure presents an quantitative analysis on the average $|K|$ values of GC (left) and the depth (cm) AvgError performance (right). Notice that stereo approaches with the lowest GC values achieve the lowest AvgError. We do not measure AvgError for monocular approaches since they usually provide relative depth.}
    \label{fig:synthetic_dataset_avgerr_x_curvature_analysis}
\end{figure}

\begin{figure}[!ht]
    \begin{center}
    \includegraphics[width=1.0\columnwidth]{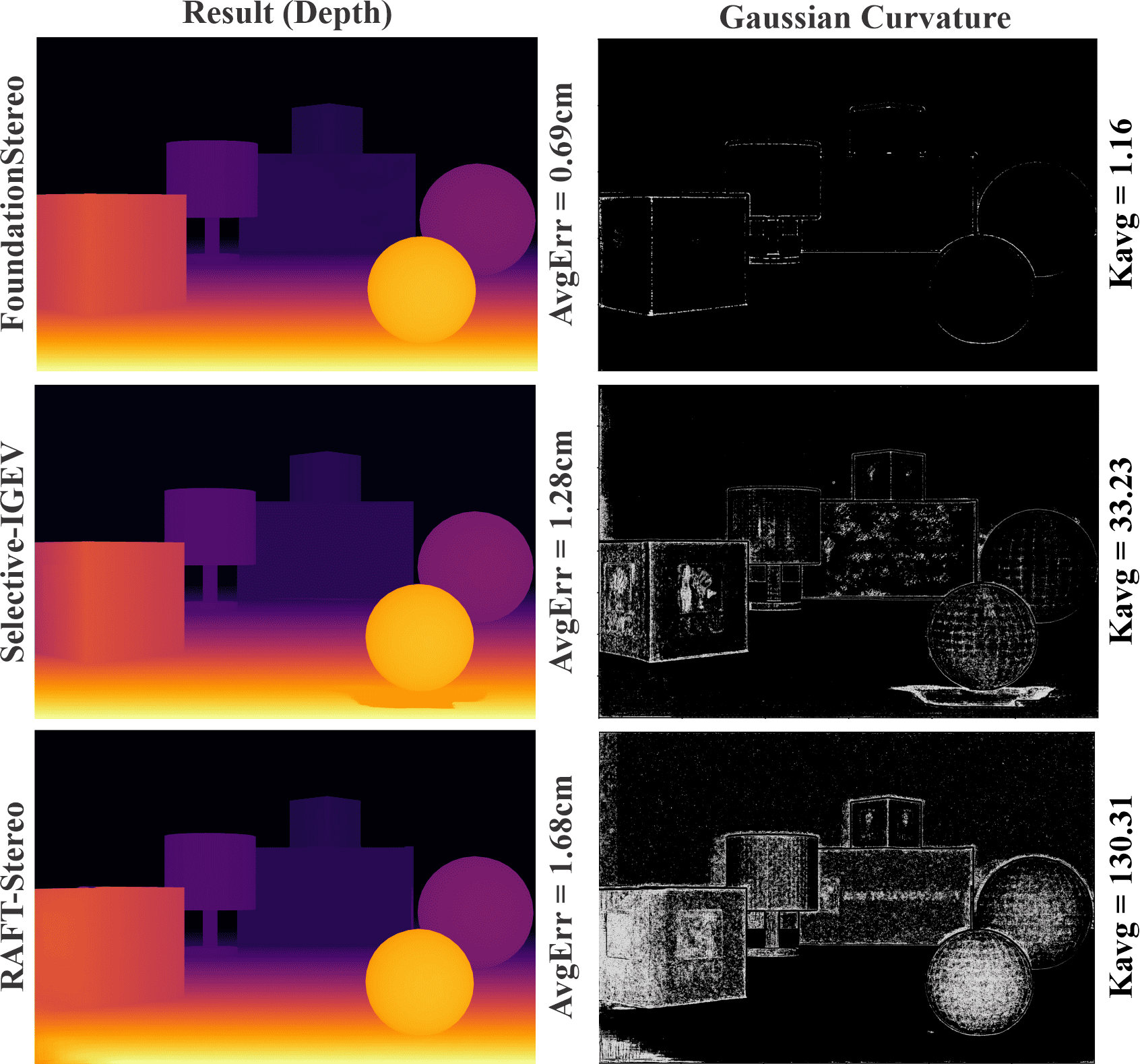}
    \end{center}
        \caption{This figure visually demonstrates the superiority of FoundationStereo in reconstructing the 3D scene with depth error inferior to 1cm in average, and low GC, preserving 3D geometry more consistently than other stereo techniques.}
    \label{fig:sota_stereo_GC}
\end{figure}

We do not report average error for DepthAnythingV2 and DepthPro, as these methods provide only relative depth. Although DepthPro estimates depth in meters and predicts the focal length from an image, our investigation revealed that its depth predictions face real-scale limitations, which is expected for monocular methods. However, to the best of our knowledge, DepthPro provides the most detailed surface estimation among current MDE approaches. While DepthPro tends to minimize GC on the Middlebury dataset (see Figures~\ref{fig:middlebury_dataset_avgerr_x_curvature_analysis} and~\ref{fig:middlebury_dataset_curvature_distribution}), both DepthPro and DepthAnythingV2 reconstruct scenes with higher curvatures in our 3D synthetic scenes (see~\autoref{fig:synthetic_dataset_avgerr_x_curvature_analysis}). To visualize the results,~\autoref{fig:sota_mono_GC} shows the 3D surface and GC estimated from SOTA MDE approaches, like DepthPro and DepthAnythingV2.

\begin{figure}
    \begin{center}
    \includegraphics[width=1.0\columnwidth]{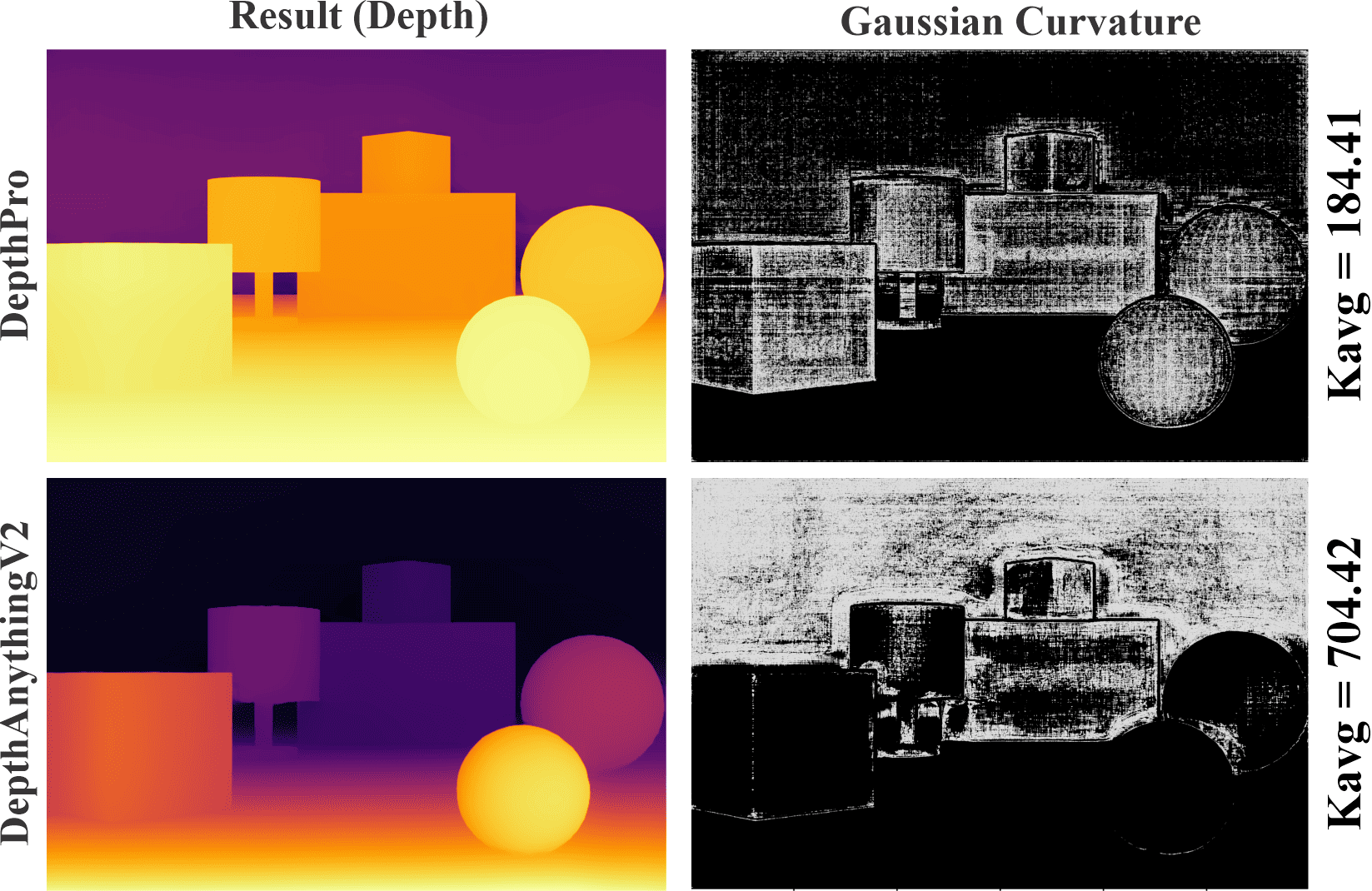}
    \end{center}
        \caption{This figure shows the 3D surfaces obtained from the SOTA MDE approaches, DepthPro and DepthAnythingV2, for the left view of the Main Scene. Observe that meaningful 3D patterns are extracted, however depth are relative and curvature presents minor inconsistencies.}
    \label{fig:sota_mono_GC}
\end{figure}
\section{Limitations and Future Directions}
\label{sec:limitations_and_future_direction}
This paper highlights Gaussian curvature as a powerful measure for analyzing 3D scenes, given its invariance to viewpoint changes and sparsity in man-made environments. However, our analysis is limited to the 15 training images from the Middlebury dataset, as it is the only benchmark that publicly provides the techniques’ results. Although we did not directly evaluate other benchmarks, the strong performance of BLMT-Stereo as the top method on Middlebury~\cite{scharstein2014high} and ETH3D~\cite{schoeps2017cvpr}, and of MonSter++ as the leading method on KITTI 2012~\cite{Geiger2012CVPR} and KITTI 2015~\cite{Menze2015CVPR}, suggests that our findings may generalize to outdoor scenes. As future work, we plan to extend our investigation to additional datasets.

In this paper, we focused on \textit{analyzing} and \textit{understanding} the role of Gaussian curvature in 3D reconstruction using SOTA stereo-vision techniques. We did not aim to develop a new algorithm at this stage. As future work, we plan to incorporate low absolute Gaussian curvature as an unsupervised metric for stereo 3D reconstruction, which also includes treating depth discontinuities, occlusions, repetitive patterns, and textureless regions.

\section{Conclusion}
\label{sec:conclusion}
Gaussian curvature (GC) plays a significant role in the reconstruction of 3D surfaces. We showed that the histogram of GC reveals a sparse distribution, suggesting that GC is a compact descriptor of surface geometry. Furthermore, we proposed that the square root of the GC magnitude can serve as a sparse loss function consistent with the observed normalized histogram distribution. In addition, we introduced a simple and efficient metric, termed Low Gaussian Curvature (LGC), which can be used as a proxy for the inverse of the Shannon entropy of the normalized histogram distribution, as higher LGC imply low entropy for the histogram.  

Our empirical evaluation of state-of-the-art methods on the Middlebury benchmark revealed that the GC normalized histograms generated by these methods approximate the high LGC values of the ground truth GC distribution both on the Middlebury dataset and on our 3D synthetic scenes. This observation suggests that very likely state-of-the-art methods do incorporate a prior that minimizes GC values in regions where data matching is not sufficient to infer a solution. However, one does not know where in the algorithm such a prior occurs nor, in analogy to modules that perform feature extraction,  how to extract a module with such a prior to be used in other applications.  

In summary, our work enhances the \textit{interpretability} and \textit{understanding} of 3D vision by highlighting Gaussian Curvature as an intrinsic geometric prior for indoor 3D surfaces. Grounded in modern deep learning data, our approach underscores the importance of 3D geometric modeling in capturing critical visual information and can guide the development of next-generation vision systems.

As a possible immediate consequence of this study, LGC could be used as a quality measure in multiple 3D reconstruction modalities, including stereo-vision, monocular depth estimation, and, by inference, structure from motion.

\section*{Acknowledgements}
This study was financed in part by the Coordenação de Aperfeiçoamento de Pessoal de Nível Superior – Brasil (CAPES) – Finance Code 001.
{
    \small
    \bibliographystyle{ieeenat_fullname}
    \bibliography{main}
}
\clearpage
\setcounter{page}{1}
\maketitlesupplementary

\section{Introduction} \label{supp_sec:intro}

This supplementary material provides a deeper exploration of the concepts, methods, and results introduced in the main paper. While the main text presents a concise overview of our findings, certain theoretical insights and experimental details require further elaboration to fully support our claims and offer transparency in our methodology. This document is intended to complement the main paper by offering readers a more comprehensive understanding of the mathematical foundations and empirical behavior of state-of-the-art (SOTA) techniques with respect to Gaussian Curvature analysis in depth estimation.

In \autoref{supp_sec:GeometricalSupplemental}, we describe the methods used to estimate Gaussian Curvature from discrete depth data and provide a formal derivation linking curvature minimization to the $L^0$ loss. This section is essential for grounding our curvature analysis in a solid mathematical framework. Then, in \autoref{supp_sec:more_experiments}, we expand our evaluation of SOTA techniques by presenting additional experimental results on both the Middlebury and our 3D synthetic scenes. These extended results not only reinforce the findings of the main paper but also reveal deeper patterns and behaviors that are critical for interpreting the performance of modern depth estimation methods.

\section{Geometrical Supplemental Material} \label{supp_sec:GeometricalSupplemental}

The material below are known geometrical properties (see \cite{Hardy52,do2016differential}) that we present here to "refresh" the interested reviewer, just in case. First we describe the two fundamental forms used to compute the Gaussian curvature (GC) and then we prove the sparsity result associated with the GC measure. 

\subsection{Methods to estimate Gaussian curvature from data}

For completion, we expand here the Gaussian curvature formula:

\begin{align}
    K={\frac {\det(\mathrm {I\!I} )}{\det(\mathrm {I} )}}={\frac {LN-M^{2}}{EG-F^{2}}}\, . 
\end{align}
where the fundamental forms I and II can be obtained as follows  

\paragraph{The first fundamental form I:} Let $\mathbf{Z}(u, v)$ be a parametric surface and let $\mathbf{Z}_u(u, v),\mathbf{Z}_v(u, v)$ denote the partial derivatives that are independent tangent vectors to the surface. Then the inner product of two tangent vectors is
\begin{align} 
&\mathrm {I} (a \mathbf{Z}_{u}+b\mathbf{Z}_{v},c\mathbf{Z}_{u}+d\mathbf{Z}_{v}) \\ &=ac\langle \mathbf{Z}_{u},\mathbf{Z}_{u}\rangle +(ad+bc)\langle \mathbf{Z}_{u},\mathbf{Z}_{v}\rangle +bd\langle \mathbf{Z}_{v},\mathbf{Z}_{v}\rangle \nonumber \\
&=\begin{pmatrix}
    a & b
\end{pmatrix}\begin{pmatrix}E&F\\F&G\end{pmatrix}\begin{pmatrix}
    c \\ d
\end{pmatrix}\, ,
\end{align}
where E, F, and G are the coefficients of the first fundamental form. 
\paragraph{The second fundamental form II:}
The vector normal to the surface is given by 
\begin{align} 
\mathbf {n} ={\frac {\mathbf{Z}_u(u, v) \times \mathbf{Z}_v(u, v)}{|\mathbf{Z}_u(u, v) \times \mathbf{Z}_v(u, v)|}}\,.
\end{align}
The second fundamental form is then written as

\begin{align} \mathrm {I\!I} =\begin{pmatrix}
    du & dv
\end{pmatrix}\begin{pmatrix}L&M\\M&N\end{pmatrix}\begin{pmatrix}
    du \\ dv
\end{pmatrix}\,,\end{align}
where 
\begin{align} 
L=\mathbf{Z} _{uu}\cdot \mathbf {n} \,,\quad M=\mathbf{Z} _{uv}\cdot \mathbf {n} \,,\quad N=\mathbf{Z} _{vv}\cdot \mathbf {n} \,.
\end{align}

\subsection{Proof of GC minimization association with $L^0$ Loss}

The result of Equation 6 is known and quite important and so we repeat here in a form of a known theorem $\sqrt{|\kappa_1 \, \kappa_2|} = \lim_{p\rightarrow 0} 
\left( \frac{1}{2} (|\kappa_1|^p + |\kappa_2|^p)\right)^{\frac{1}{p}} $. We prove this theorem by breaking it into three lemmas before the final proof. 

\begin{lemma}
For $r>1$, $\left (\frac{|\kappa_1|^p + |\kappa_2|^p}{2}\right)^{r}\le  \frac{|\kappa_1|^{p\, r} + |\kappa_2|^{p\, r}}{2}$\, ,
 \label{lemma:r}
\end{lemma}
\begin{proof}
Define $a_1=|\kappa_1|^p\ge 0$ and $a_2=|\kappa_2|^p \ge 0$, and  since $r\ge 1$, then  $(a_1+ a_2)^{r}$ is convex. From Jensen inequality we then have $(\frac{a_1+ a_2}{2})^{r} \le \frac{1}{2}(a_1^{r}+ a_2^{r})$\, . Replacing back $a_1=|\kappa_1|^p$ and $a_2=|\kappa_2|^p $ completes the proof. 
\end{proof}

\begin{lemma}
 For $1\ge q\ge p\ge 0$,  $\left(\frac{|\kappa_1|^p + |\kappa_2|^p}{2}\right)^{\frac{1}{p}}\le \left(\frac{|\kappa_1|^q + |\kappa_2|^q}{2}\right )^{\frac{1}{q}}$ \, ,
 \label{lemma:pq}
\end{lemma}
\begin{proof}
Replacing $r$ by $\frac{q}{p}\ge 1$ in lemma~\ref{lemma:r}
$\left(\frac{|\kappa_1|^p + |\kappa_2|^p}{2}\right)^{\frac{q}{p}}\le  \frac{|\kappa_1|^q + |\kappa_2|^q}{2}$. Then taking the $q$ root on both sides does not change the order of the inequality and completes the proof. 
\end{proof}

\begin{lemma}
 For $1\ge p\ge 0$, 
 \begin{align}  \frac{\log |\kappa_1| + \log |\kappa_2|}{2} \le \log  \left (\frac{|\kappa_1|^p + |\kappa_2|^p}{2}\right)^{\frac{1}{p}}    \, ,
  \end{align}
  and thus $\left (\frac{|\kappa_1|^p + |\kappa_2|^p}{2}\right)^{\frac{1}{p}}$ is bounded from below. 
 \label{lemma:bounded}
\end{lemma}
\begin{proof}
 The log function is a concave function. Thus, 
 \begin{align}
     \frac{\log |\kappa_1|^p + \log |\kappa_2|^p}{2} \le \log \left (\frac{|\kappa_1|^p + |\kappa_2|^p}{2}\right) 
 \end{align}
follows from Jensen's inequality. Since $\log |\kappa_1|^p + \log |\kappa_2|^p= p\, (\log |\kappa_1| + \log |\kappa_2|)$, then back to the Jensen's inequality $p\, \left (\frac{\log |\kappa_1| + \log |\kappa_2|}{2} \right) \le \log \left (\frac{|\kappa_1|^p + |\kappa_2|^p}{2}\right)  $ and so 

$\left (\frac{\log |\kappa_1| + \log |\kappa_2|}{2} \right) \le \frac{1}{p}\log  \left (\frac{|\kappa_1|^p + |\kappa_2|^p}{2}\right) $  completes the proof. 
\end{proof}
From lemma \eqref{lemma:pq} and lemma \eqref{lemma:bounded}, it follows that $\left (\frac{|\kappa_1|^p + |\kappa_2|^p}{2}\right)^{\frac{1}{p}}$ decreases as $p$ decreases and it is bounded from below. Therefore, it converges as $p\rightarrow 0$
.

\begin{theorem}
 \begin{align}   
 \lim_{p\rightarrow 0}\left (\frac{|\kappa_1|^p + |\kappa_2|^p}{2}\right)^{\frac{1}{p}}  = e^{\frac{1}{2}\log (|\kappa_1| \, |\kappa_2|)}  \, ,
  \end{align}
 \label{theorem:main}
\end{theorem}

\begin{proof}
 We use a  known inequality that  for $0\le p \le 1$,  $\log x \le \frac{1}{p}(x^p-1)$ . 
 Thus, for $x=\left (\frac{|\kappa_1|^p + |\kappa_2|^p}{2}\right)^{\frac{1}{p}}$  we have 
 \begin{align}
     \log \left (\frac{|\kappa_1|^p + |\kappa_2|^p}{2}\right)^{\frac{1}{p}}  \le \frac{1}{p}\left (\left (\frac{|\kappa_1|^p + |\kappa_2|^p}{2}\right)-1\right )
     \\ = \frac{1}{2} \left ( \frac{1}{p}(|\kappa_1|^p-1) +  \frac{1}{p}(|\kappa_2|^p-1) \right )
 \end{align}
 Taking the limit $p\rightarrow 0$ and using the equality  $\log x = \lim_{p\rightarrow 0}\frac{1}{p}(x^p-1)$ (use L'H\^opital rule to check) we finally obtain
 \begin{align}
    \lim_{p\rightarrow 0} \log \left (\frac{|\kappa_1|^p + |\kappa_2|^p}{2}\right)^{\frac{1}{p}}  & \le \frac{\log |\kappa_1|+\log |\kappa_2|}{2}
 \end{align}
 Thus the $\lim_{p\rightarrow 0} \log \left (\frac{|\kappa_1|^p + |\kappa_2|^p}{2}\right)^{\frac{1}{p}}$ is bounded by $\frac{\log |\kappa_1|+\log |\kappa_2|}{2}$ from above and,  by lemma~\ref{lemma:bounded}, bounded from below  by the same quantity and so 
 \begin{align}
    \lim_{p\rightarrow 0} \log \left (\frac{|\kappa_1|^p + |\kappa_2|^p}{2}\right)^{\frac{1}{p}}=\frac{\log |\kappa_1|+\log |\kappa_2|}{2}\, .
 \end{align}
 Taking the exponential from both sides completes the proof. 
\end{proof}

\section{In Depth Analysis on the SOTA approaches} \label{supp_sec:more_experiments}

In this section we present in details more experiments we have conducted for the GC analysis and understanding. In the following subsections we discuss the results for Middlebbury Dataset, and for our 3D synthetic scenes, respectively.

\subsection{Middlebury Dataset} \label{supp_subsec:middlebury_dataset}

In the main paper, we presented a normalized histogram distribution in~\autoref{fig:middlebury_dataset_curvature_distribution}, which shows the LGC metric for FoundationStereo, DepthPro, RAFT-Stereo, Selective-IGEV, and in~\autoref{fig:histogram-Gaussian-curvature-middleburry} for the ground truth (GT). We also present in~\autoref{tab:middlebury_benchmark} a ranking that includes several benchmarking metrics alongside LGC. We selected these techniques for inclusion in the main text because we had already tested their code on our 3D synthetic scenes. Additionally, they represent key categories: FoundationStereo belongs to Group A (new SOTA), RAFT-Stereo and Selective-IGEV to Group B (previous SOTA), and DepthPro is the best-performing monocular depth estimation (MDE) method in terms of depth reconstruction.

Here, in the supplementary material (see~\autoref{fig:middlebury_dataset_curvature_distribution_suppmaterial}), we take advantage of the additional space to extend our analysis to all evaluated techniques. We emphasize our earlier observation: Group A approaches not only achieve the lowest average disparity errors in the Middlebury ranking but also tend to minimize Gaussian Curvature—thus maximizing the LGC metric. Notably, FoundationStereo, MonoStereo, LG-Stereo, and DEFOM-Stereo exhibit LGC values above 65\%, while Selective-IGEV, RAFT-Stereo, DLNR, and CREStereo show LGC values below 36\%.

\begin{figure*}[!ht]
    \begin{center}
    \includegraphics[width=\textwidth]{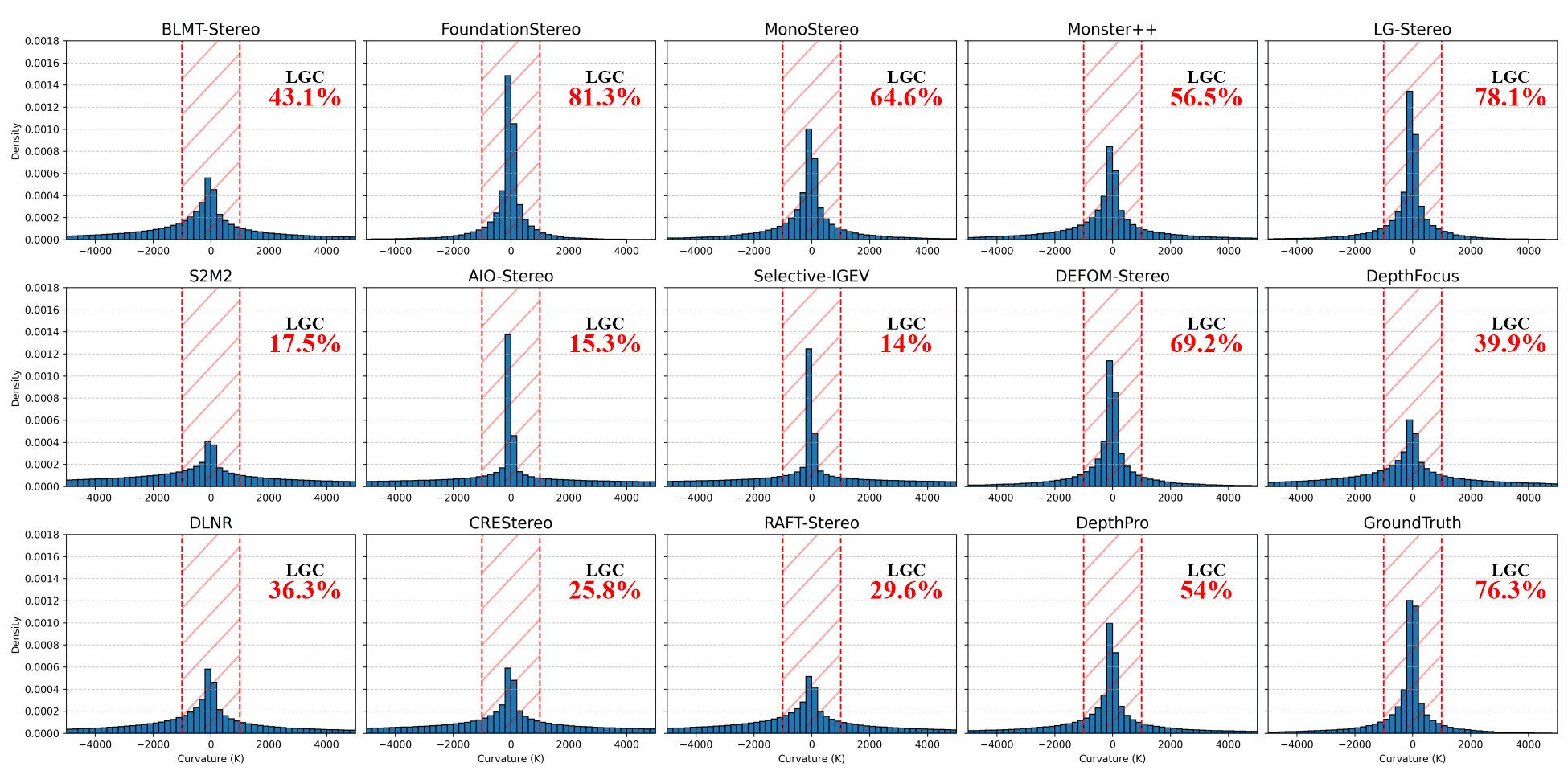}
    \end{center}
        \caption{\textbf{Curvature Distribution:} each plot presents a normalized histogram of the GC distribution for all 15 training images from the Middlebury dataset. We discarded the highest 20\% of $|K|$ values, and plotted the remaining $K$ values within $[-5,000, 5,000]m^{-2}$ in 50 bins uniformly distributed.}
    \label{fig:middlebury_dataset_curvature_distribution_suppmaterial}
\end{figure*}

We also present a detailed analysis of the Middlebury ranking and Gaussian Curvature for each of the 15 training images. While the main paper shows the average disparity error aggregated across all 15 images, \autoref{fig:middlebury_dataset_ranking_per_image} provides a per-image breakdown, showing each method’s ranking position alongside its average disparity error for each individual image. In \autoref{fig:middlebury_dataset_curvature_per_image}, we show the average Gaussian Curvature for each of the 15 training images. 

It can be observed in~\autoref{fig:middlebury_dataset_ranking_per_image} that the Group A approaches consistently achieve the lowest average disparity error across all images. Furthermore, \autoref{fig:middlebury_dataset_curvature_per_image} shows that the Gaussian Curvature values produced by Group A approaches tend to have lower absolute magnitudes ($|K|$) compared to those of Group B approaches.

\begin{figure*}[!ht]
    \begin{center}
    \includegraphics[width=\textwidth]{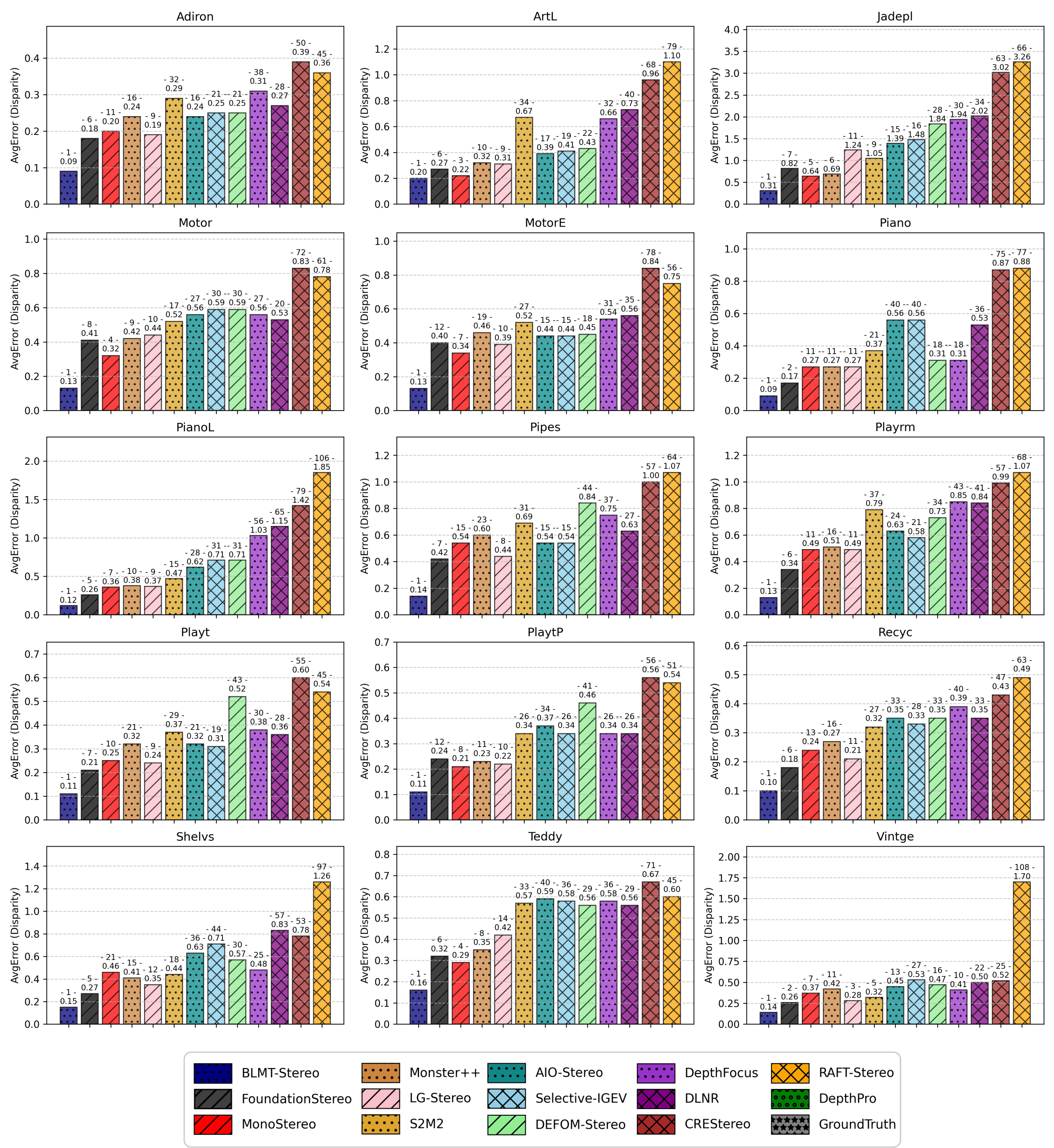}
    \end{center}
        \caption{\textbf{Middlebury Ranking:} Average disparity error (AvgError) per image (see \autoref{supp_subsec:middlebury_dataset} for more information).}
    \label{fig:middlebury_dataset_ranking_per_image}
\end{figure*}

\begin{figure*}[!ht]
    \begin{center}
    \includegraphics[width=0.90\textwidth]{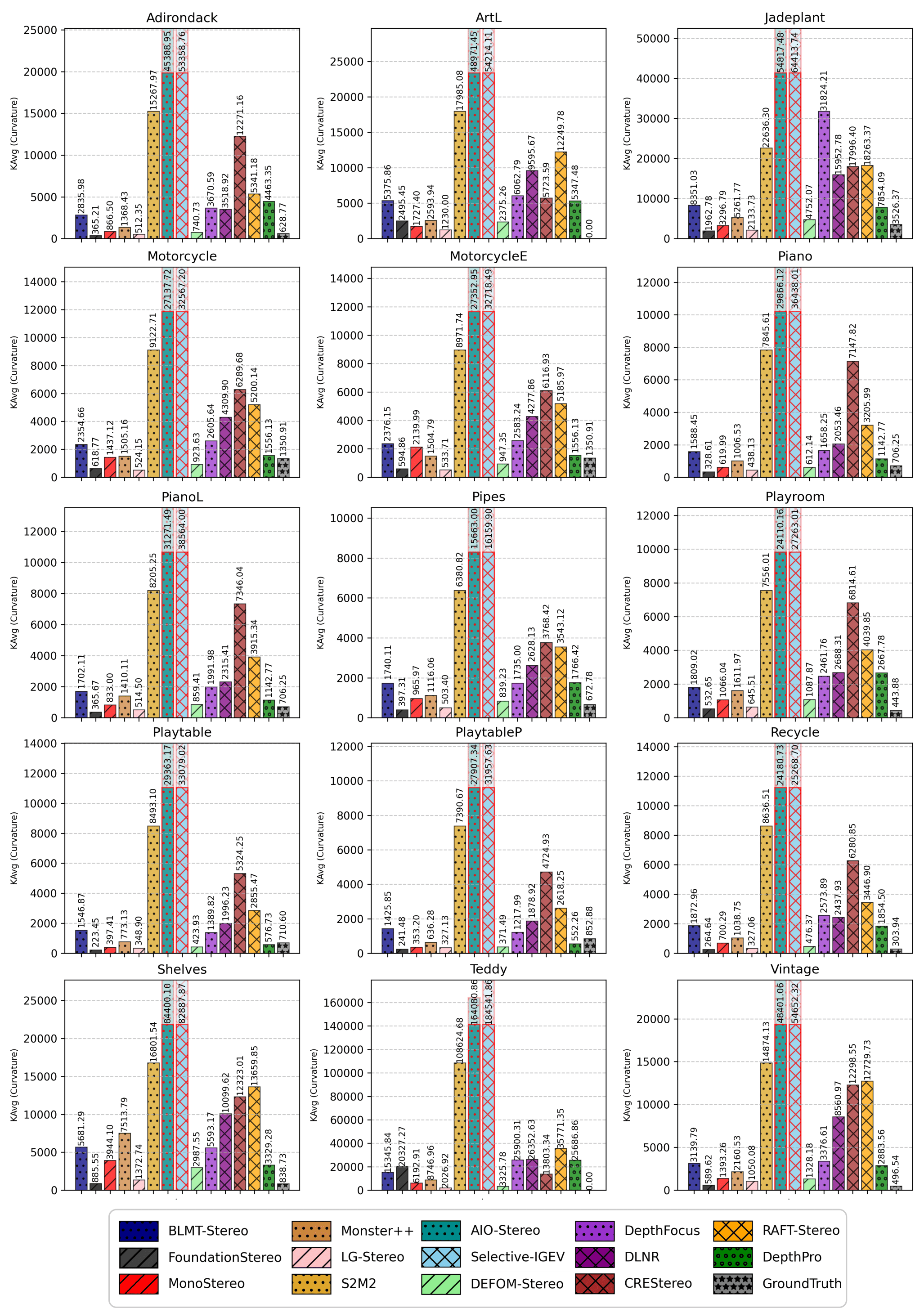}
    \end{center}
        \caption{\textbf{Middlebury Curvature:} Average Absolute Gaussian Curvature (Avg$|K|$) per image (see \autoref{supp_subsec:middlebury_dataset} for more information).}
    \label{fig:middlebury_dataset_curvature_per_image}
\end{figure*}

\subsection{3D Synthetic Scenes} \label{supp_subsec:synthetic_dataset}

In the following experiments, we analyze quantitatively and qualitatively the depth and curvature results for each of the 5 scenes: Box\_Rotation\_45, Box\_Rotation\_90, Cylinder, Sphere, and MainScene. As already mentioned, we obtained the source code of: Group A) FoundationStereo; Group B) Selective-IGEV, and RAFT-Stereo; and MDE) DepthPro, and DepthAnythingV2.

The results of each aforementioned approach are shown in Figures~\ref{fig:synthetic_dataset_FoundationStereo},~\ref{fig:synthetic_dataset_Selective-IGEV},~\ref{fig:synthetic_dataset_RAFT-Stereo},~\ref{fig:synthetic_dataset_DepthPro}, and~\ref{fig:synthetic_dataset_DepthAnythingV2}, respectively. For the Group A and Group B approaches, which estimate the disparity of a scene, we computed depth using the equation $Depth = \frac{f \cdot b}{d}$, where $f$ is the focal length in pixels, $b$ is the baseline in meters, and $d$ is the disparity value in pixels. Note that no pixel offset correction is needed ($doffs = 0$), as our simulated environment ensures that the left and right images are rectified. The resulting depth for each technique is shown in the first column of its respective figure.

For the Group A and Group B approaches, we also computed the difference between each method’s estimated depth, denoted as $D_{\text{technique}}$, and the expected ground truth ($GT$). The second column in each technique’s figure presents a visual comparison of the difference $GT - D_{\text{technique}}$ for each scene in the 3D synthetic scenes. Since depth values are strictly positive, this difference can yield both positive (red) and negative (blue) values. A positive value ($GT - D_{\text{technique}} > 0$, shown in red) indicates that the method predicted a depth further ahead than it should have. Conversely, a negative value ($GT - D_{\text{technique}} < 0$, shown in blue) indicates that the method predicted a depth further back than the correct position.

For all techniques, we computed the Gaussian Curvature, which is shown in the third column of each figure. To facilitate the analysis, we masked regions where $|K| < \text{threshold}$ (with a threshold of 1000~$\mathrm{m}^{-2}$), displaying them in black. This masking step allows us to focus on regions with higher curvatures and better understand the behavior of each method.

All relevant insights and observations are described in the figure's captions to support and streamline the reading process. We also omitted some parts of the figures for Blind Review.

\begin{figure*}[!ht]
    \begin{center}
    \includegraphics[width=0.95\textwidth]{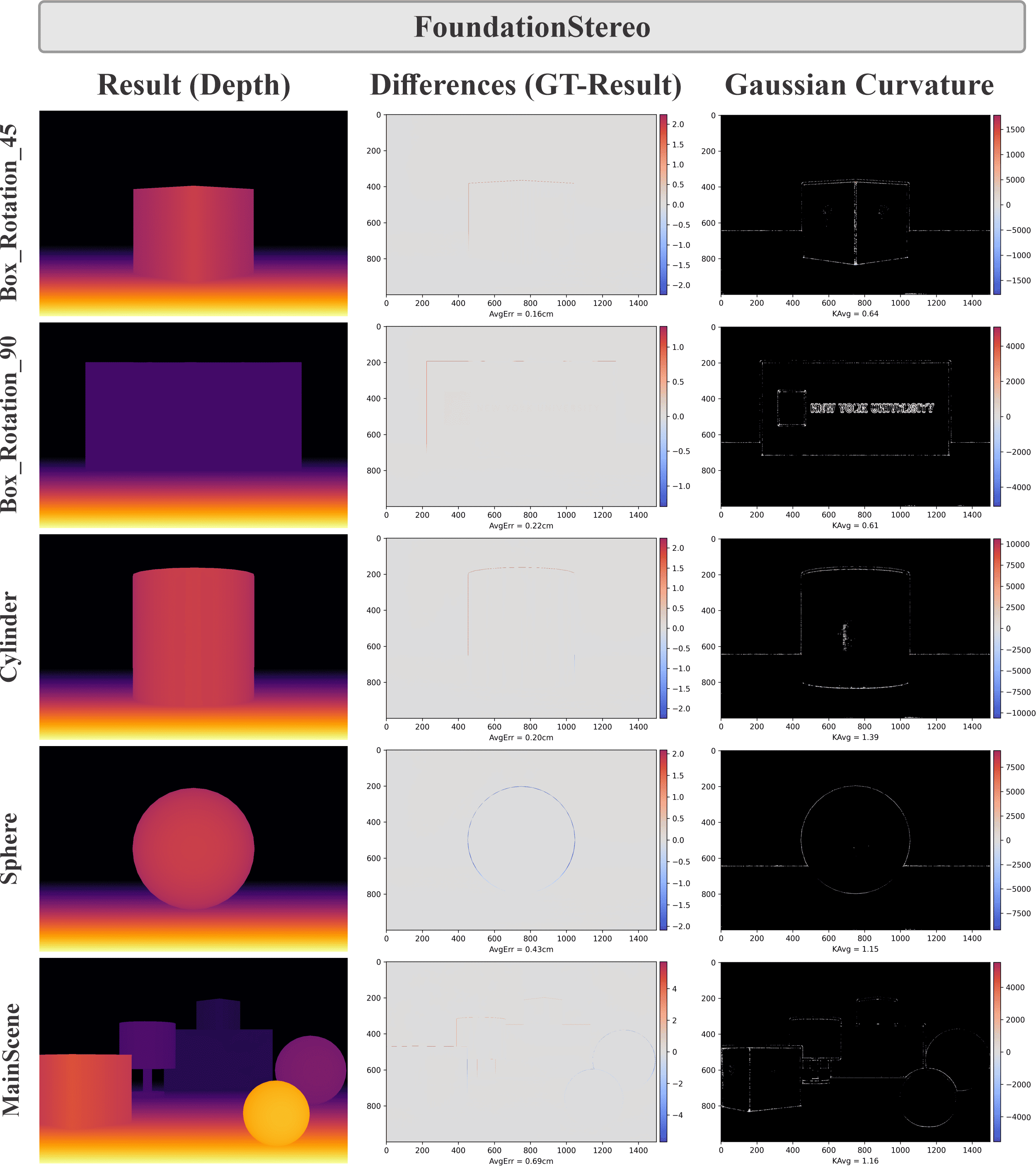}
    \end{center}
        \caption{Results of \textbf{FoundationStereo} on our 3D synthetic scenes. FoundationStereo achieved a depth error of less than one centimeter across all five scenes. Additionally, it estimated the lowest Gaussian Curvature among all evaluated approaches. Notably, regions with $|K| > 1000~\mathrm{m}^{-2}$ appear only near edges. Once again, FoundationStereo demonstrates its position as the best-performing method in both the Middlebury dataset and our 3D synthetic scenes.}
    \label{fig:synthetic_dataset_FoundationStereo}
\end{figure*}

\begin{figure*}[!ht]
    \begin{center}
    \includegraphics[width=0.95\textwidth]{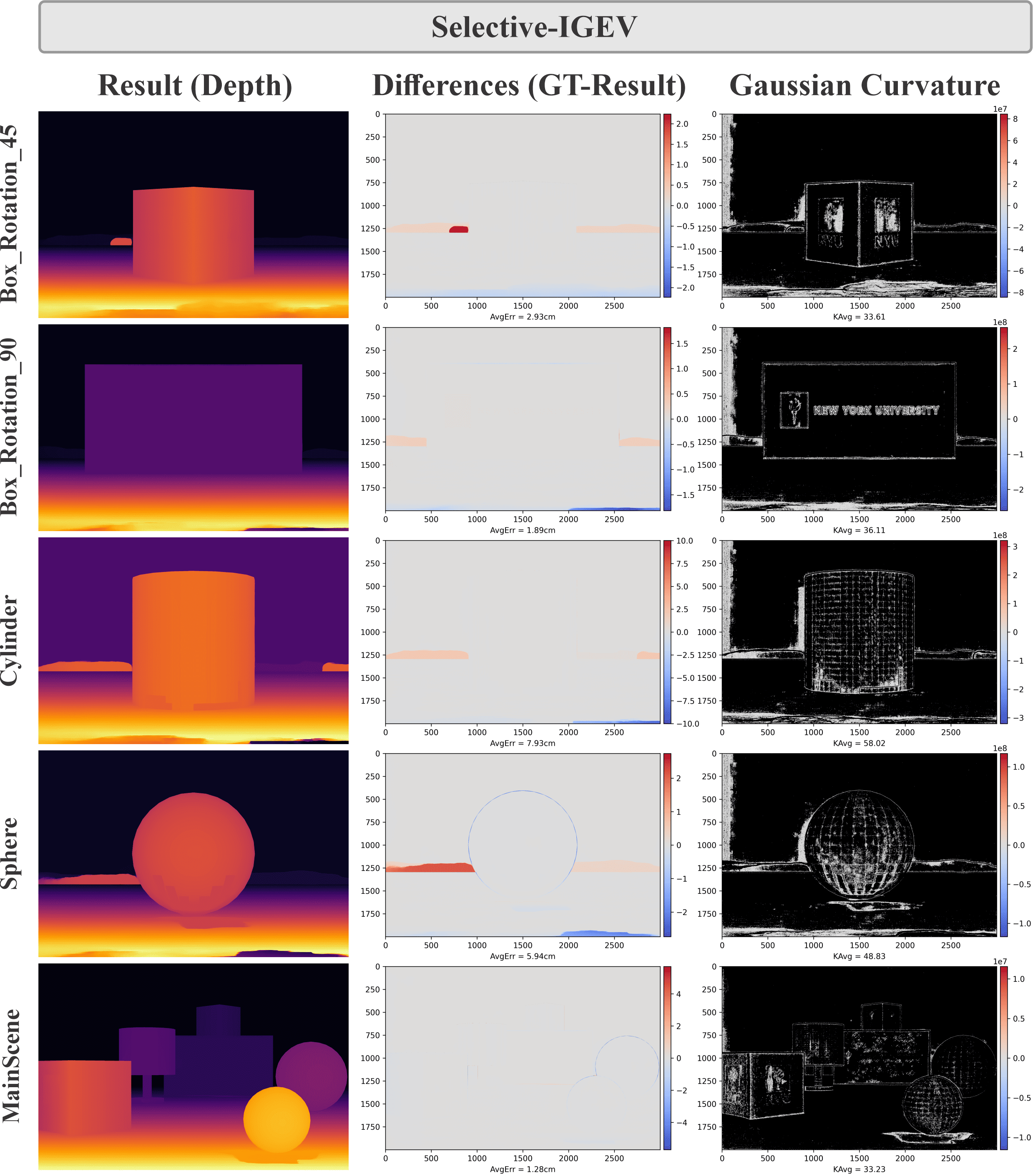}
    \end{center}
        \caption{Results of \textbf{Selective-IGEV} on our 3D synthetic scenes. Selective-IGEV showed some inaccuracies in depth estimation and produced higher $|K|$ values, as observed in the curvature plots. While it exhibited the highest Gaussian Curvature among all methods in the Middlebury dataset, in our 3D synthetic scenes it estimated lower curvature values than RAFT-Stereo (see~\autoref{fig:synthetic_dataset_avgerr_x_curvature_analysis} in the main paper).}
    \label{fig:synthetic_dataset_Selective-IGEV}
\end{figure*}

\begin{figure*}[!ht]
    \begin{center}
    \includegraphics[width=0.95\textwidth]{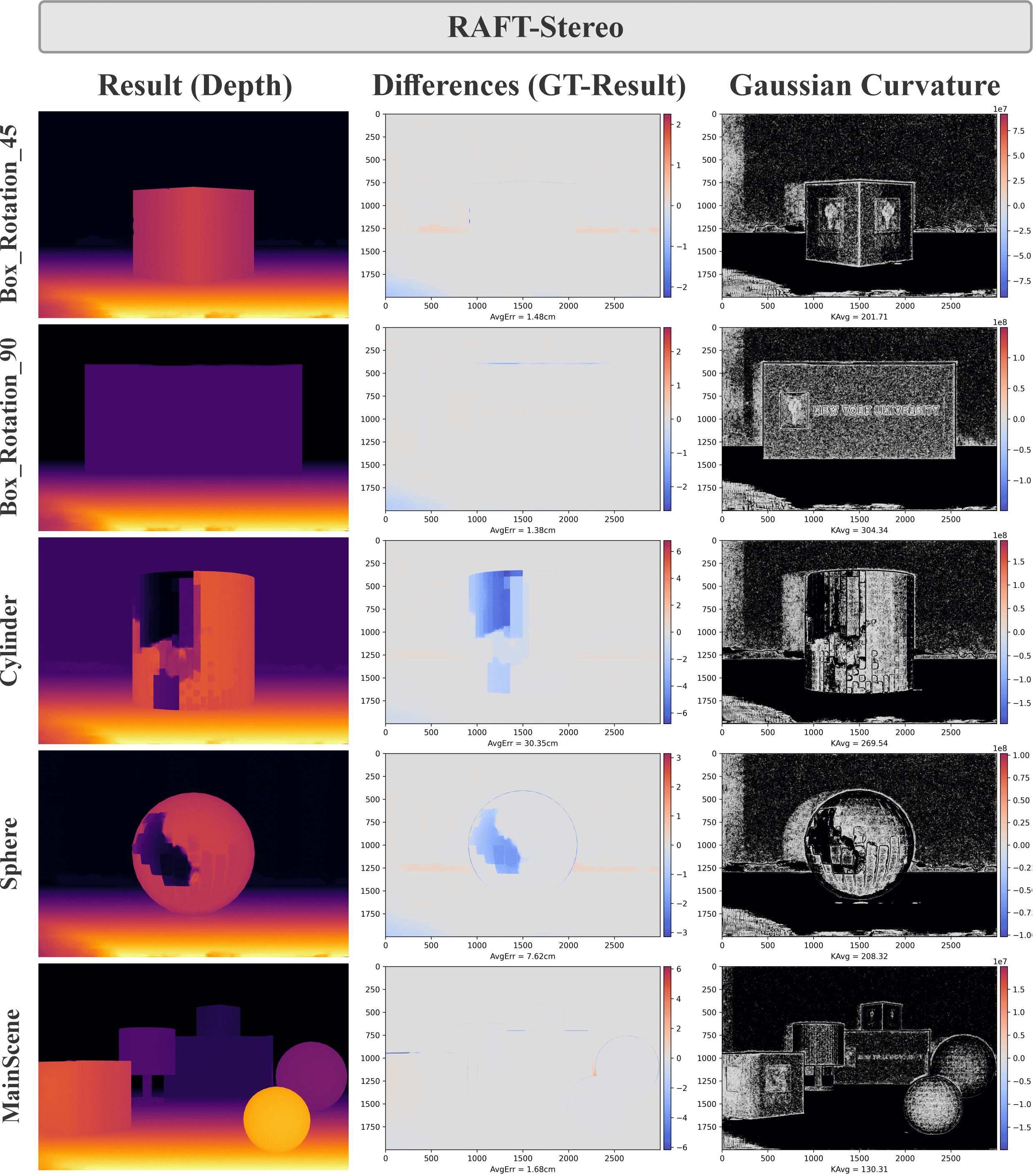}
    \end{center}
        \caption{Results of \textbf{RAFT-Stereo} on our 3D synthetic scenes. RAFT-Stereo exhibited several inconsistencies in depth estimation. Notably, the Cylinder and Sphere scenes showed average errors exceeding 30~cm and 7~cm, respectively. Despite these limitations, RAFT-Stereo performed reasonably well in the remaining scenes. However, it produced the highest Gaussian Curvature among all evaluated approaches in our 3D synthetic scenes, including both Group A and Group B methods.}
    \label{fig:synthetic_dataset_RAFT-Stereo}
\end{figure*}

\begin{figure*}[!ht]
    \begin{center}
    \includegraphics[width=0.95\textwidth]{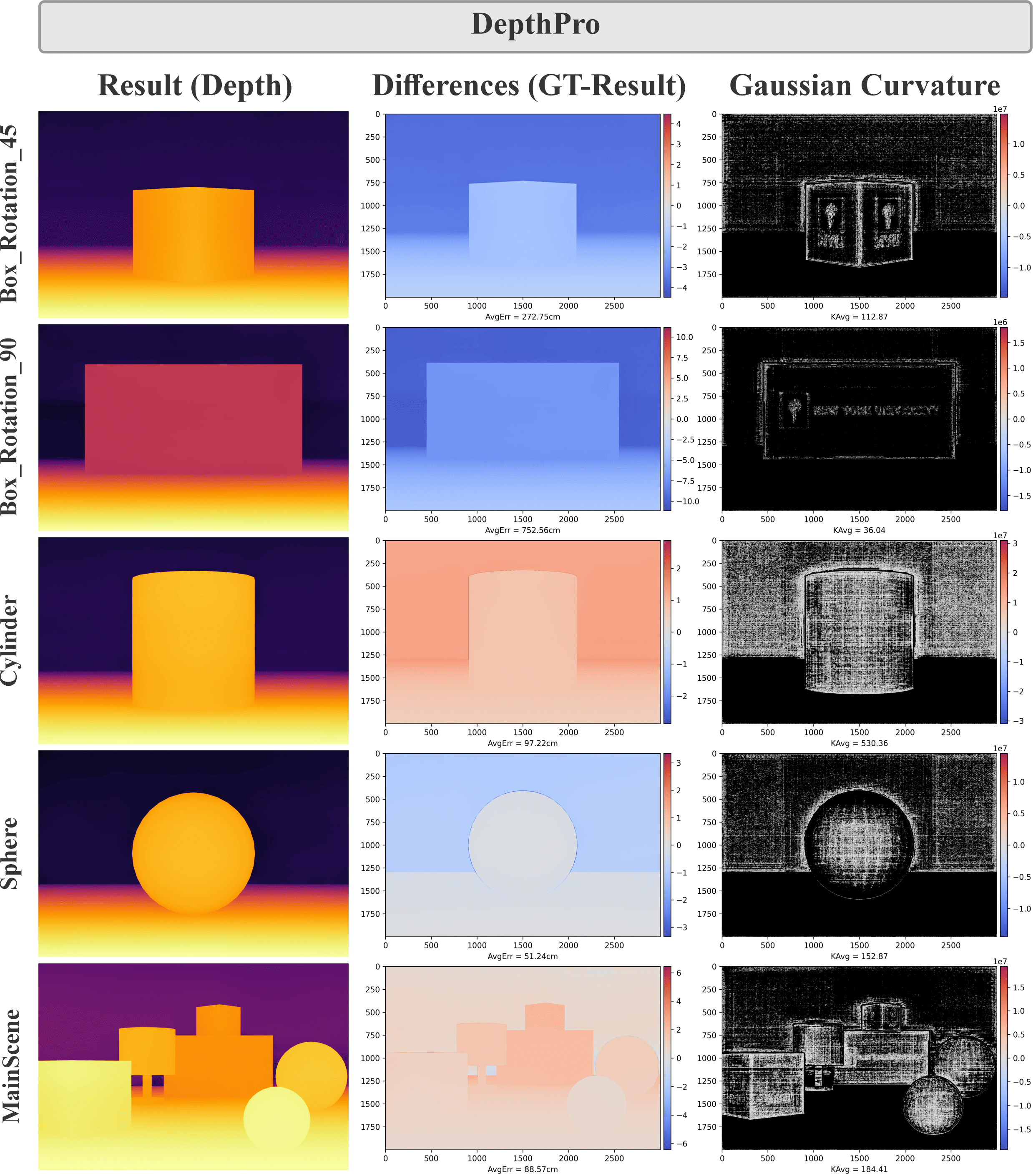}
    \end{center}
        \caption{Results of \textbf{DepthPro} on our 3D synthetic scenes. DepthPro estimates depth in meters and also predicts focal length from a single image. In the main paper, we do not report DepthPro's average depth error due to observed limitations related to real-scale accuracy. As shown in the depth difference maps, DepthPro's minimum and maximum errors range from 51~cm to 752~cm. On the other hand, DepthPro produces reasonably accurate Gaussian Curvature estimates for flat surfaces such as the ground, boxes, and walls. As discussed in the paper, DepthPro achieves an LGC of approximately 54\% in Middlebury Dataset, positioning it between the Group A methods, which tend to produce higher LGC values, and Group B methods, which generally show lower LGC values.}
    \label{fig:synthetic_dataset_DepthPro}
\end{figure*}

\begin{figure*}[!ht]
    \begin{center}
    \includegraphics[width=0.65\textwidth]{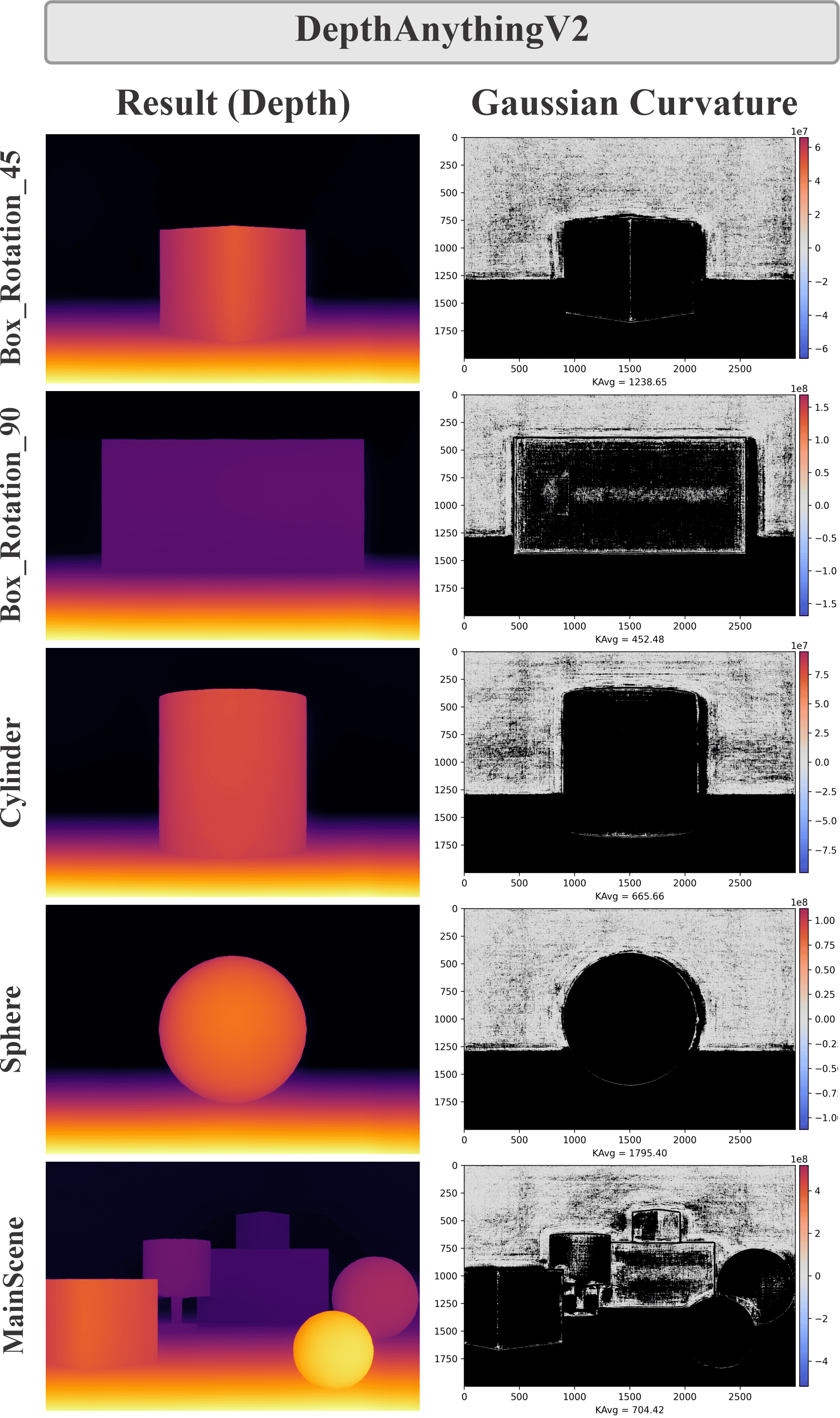}
    \end{center}
        \caption{Results of \textbf{DepthAnythingV2} on our 3D synthetic scenes. DepthAnythingV2 provides relative depth; therefore, we do not report its average depth error. Qualitatively, the depth reconstruction appears consistent, correctly preserving the relative positioning of objects in the scene. However, DepthAnythingV2 exhibits the highest average Gaussian Curvature in our 3D synthetic scenes, indicating limitations in capturing intrinsic geometric relationships. As seen in the curvature plots, high $|K|$ values are distributed across the entire scene—especially on the wall—despite the expected curvature being nonzero only near edges and on spherical surfaces.}
    \label{fig:synthetic_dataset_DepthAnythingV2}
\end{figure*}

\end{document}